\documentclass[11pt, a4paper, logo, onecolumn,copyright]{arxiv}

\usepackage[authoryear, sort&compress, round]{natbib}
\title{Uncertainty Quantification in LLM Agents: Foundations, Emerging Challenges, and Opportunities}

\usepackage{tabularx}
\usepackage[utf8]{inputenc} 
\usepackage[T1]{fontenc}    
\usepackage{hyperref}       
\usepackage{url}            
\usepackage{booktabs}       
\usepackage{nicefrac}       
\usepackage{microtype}      
\usepackage{amsmath}
\usepackage{graphicx}
\usepackage{multicol}
\usepackage{siunitx}
\usepackage{tcolorbox}
\tcbuselibrary{breakable, skins}
\usepackage{array}
\usepackage[nameinlink]{cleveref}
\usepackage{bbm}
\usepackage{multirow}
\usepackage{subfig}
\usepackage{soul}
\usepackage{floatrow}
\usepackage{float}
\usepackage{makecell}
\usepackage{wrapfig}
\usepackage{blindtext}
\usepackage{tablefootnote}
\usepackage{amsfonts}
\usepackage[flushleft]{threeparttable}
\usepackage{colortbl}
\usepackage{bbding}
\usepackage{lipsum}
\usepackage{amssymb}
\usepackage{pifont}
\newcommand{\cmark}{\ding{51}}%
\newcommand{\xmark}{\ding{55}}%
\theoremstyle{plain}

\newtheorem{lemma}{Lemma}

\theoremstyle{definition}
\newtheorem{definition}{Definition}

\theoremstyle{remark}

\usepackage{thm-restate}

\usepackage{xcolor}
\usepackage{colortbl}
\definecolor{lg}{gray}{0.9}
\definecolor{dg}{RGB}{0,150,0}
\definecolor{dr}{RGB}{139,0,0}
\DeclareSymbolFont{extraup}{U}{zavm}{m}{n}
\DeclareMathSymbol{\varheart}{\mathalpha}{extraup}{86}
\DeclareMathSymbol{\vardiamond}{\mathalpha}{extraup}{87}
\usepackage{xspace}
\usepackage{floatrow}

\newcommand{\draftonly}[1]{#1}
\newcommand{\eat}[1]{}

\renewcommand{\draftonly}[1]{}
\definecolor{darkgreen}{RGB}{0, 102, 0}

\crefformat{section}{\S#2#1#3}

\usepackage{amsmath,amsfonts,bm}

\def\eqref#1{equation~\ref{#1}}

\def\1{\bm{1}}

\DeclareMathAlphabet{\mathsfit}{\encodingdefault}{\sfdefault}{m}{sl}
\SetMathAlphabet{\mathsfit}{bold}{\encodingdefault}{\sfdefault}{bx}{n}

\author[$1$]{Changdae Oh}
\author[$1$]{Seongheon Park}
\author[$2$]{To Eun Kim}
\author[$1$]{Jiatong Li}
\author[$1$]{Wendi Li}
\author[$1$]{Samuel Yeh}
\author[$3$]{Xuefeng Du}
\author[$4$]{\mbox{Hamed Hassani}}
\author[$5$]{Paul Bogdan}
\author[$6$]{Dawn Song}
\author[$1$]{Sharon Li}

\affil[ 1]{University of Wisconsin--Madison}
\affil[ 2]{Carnegie Mellon University}
\affil[ 3]{Nanyang Technological University}
\affil[ 4]{\mbox{University of Pennsylvania}}
\affil[ 5]{University of Southern California}
\affil[ 6]{University of California, Berkeley}

\correspondingauthors{\{\texttt{changdae}, \texttt{sharonli}\}\texttt{@cs.wisc.edu}}
\artifacts{https://huggingface.co/datasets/changdae/tau2-uq-artifacts}
\projectpage{https://agentuq.github.io/}

\begin{abstract}
Uncertainty quantification (UQ) for large language models (LLMs) is a key building block for safety guardrails of daily LLM applications. Yet, even as LLM agents are increasingly deployed in highly complex tasks, most UQ research still centers on single-turn question-answering. We argue that \textit{UQ research must shift to realistic settings with interactive agents, and that a new principled framework for agent UQ is needed}. This paper presents three pillars to build a solid ground for future agent UQ research: (\textbf{1. Foundations}) We present the first general formulation of agent UQ that subsumes broad classes of existing UQ setups; (\textbf{2. Challenges}) We identify four technical challenges specifically tied to agentic setups---selection of uncertainty estimator, uncertainty of heterogeneous entities, modeling uncertainty dynamics in interactive systems, and lack of fine-grained benchmarks---with numerical analysis on a real-world agent benchmark, $\tau^2$-bench; (\textbf{3. Future Directions}) We conclude with noting on the practical implications of agent UQ and remaining open problems as forward-looking discussion for future explorations.
\end{abstract}

\begin{document}
\addtocontents{toc}{\protect\setcounter{tocdepth}{-1}}

\maketitle

\section{Introduction} \label{sec:intro}
\begin{figure*}
    \centering
    \vspace{-0.5em}
    \includegraphics[width=\linewidth]{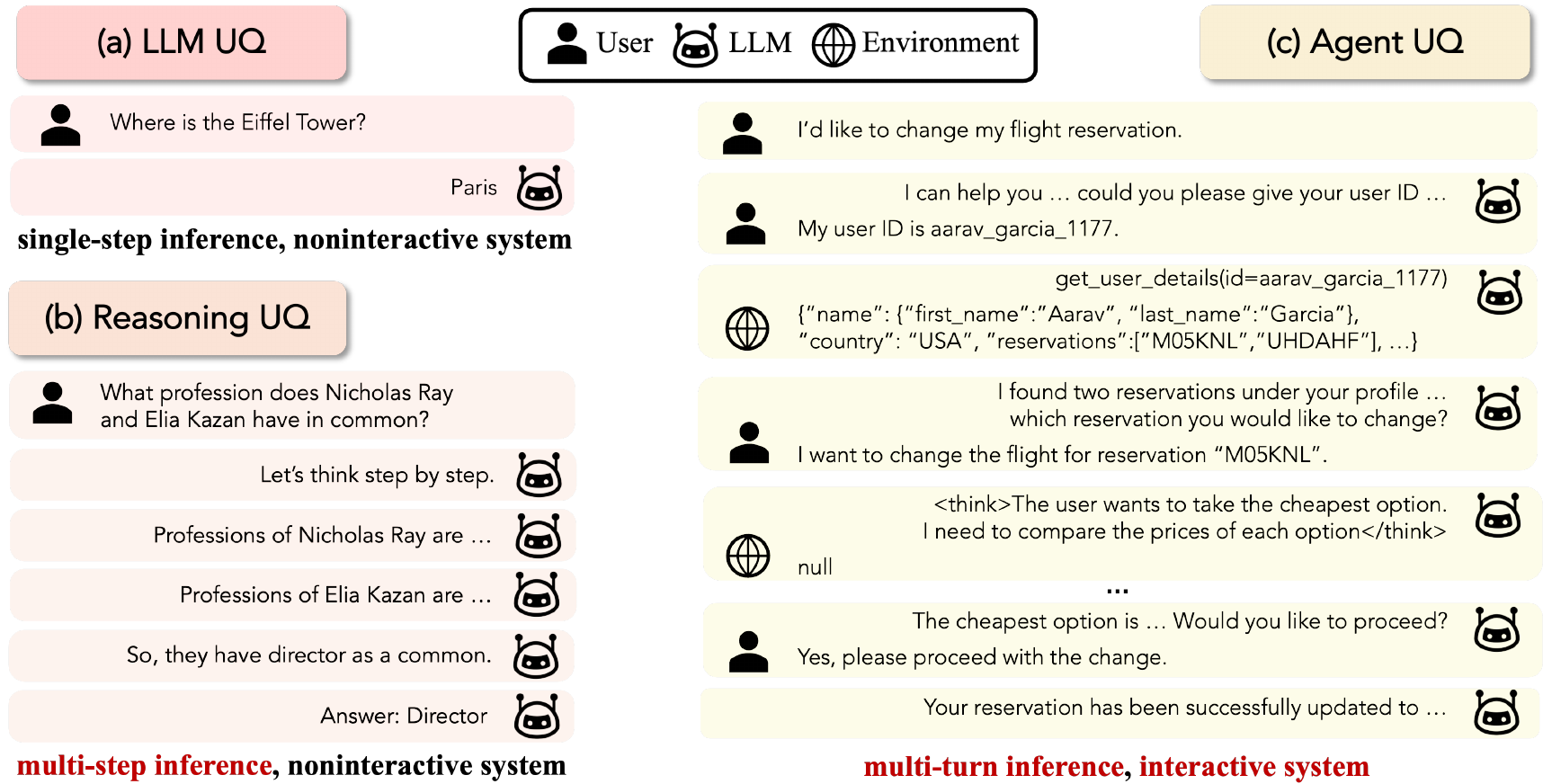}
    \caption{\textbf{Comparison between UQ setups.} Traditional LLM UQ (a) measures the uncertainty of an answer given a question, whereas the UQ for LLM reasoning (b) expands the problem by considering multi-step responses rather than a single response. Agent UQ (c) goes further by considering continual interactions between agent and user/environment across the trajectory, making it a multi-turn, interactive inference setup (example sourced from $\tau^2$-bench Airline; \cite{barres2025tau}).}
    \label{fig:illu_traj}
    \vspace{-0.5em}
\end{figure*}

LLM-based agents operating in open-world environments take actions that have real consequences: making costly bookings, modifying databases, or issuing irreversible commands~\citep{openai2025agent,google2025agent,anthropic2026cowork} at the next level of autonomy~\citep{Steinberger2026OpenClaw}. In such settings, failures are no longer limited to incorrect text generation: agents may act prematurely under unresolved ambiguity, propagate errors across long interaction trajectories, or commit to outcomes that are costly or difficult to undo. 
To be deployed responsibly, agents must be able to assess and act upon the likelihood of failure~\citep{kochenderfer2015decision}. This makes uncertainty quantification (UQ) an urgent and central requirement for agentic systems.

Despite the importance, most existing UQ research treats LLMs as static oracles: a system examined in isolation, prompted once, and evaluated by the uncertainty of a single response or a chain of responses (Figure~\ref{fig:illu_traj}, left). These methods implicitly assume a \emph{static system} in which no new information is acquired after the initial prompt~\citep{malinin2021uncertainty,kadavath2022language,farquhar2024detecting}. Under this assumption, uncertainty is treated as a point-wise estimate or uni-directional propagation process. While appropriate for single-turn inference or non-interactive multi-step reasoning, \textit{they do not meet the growing needs of UQ in agentic settings featured by its long-horizon interactive environment with heterogeneous entities.}

LLM agents in this setup operate with feedback from users and system components through multi-turn interactions to complete a complex task (Figure~\ref{fig:illu_traj}, right). For example, a flight-booking agent needs to decide whether to finalize a reservation or ask follow-up questions about dates, budget, or layover preferences. Early uncertainty should prompt information-seeking actions, while subsequent dialogue and database queries can resolve ambiguities and reduce uncertainty, enabling confident execution. This deviates significantly from the setup considered in classic LLM UQ, as the agent now spans a long interaction log updated over multiple turns that involves messages from other entities, and it also gets additional information from environments. This gap raises questions about the applicability of existing UQ methods to agentic setups.

In this paper, we argue that building a trustworthy LLM-driven agentic system requires a shift in focus: from pointwise uncertainty of the final answer to structured uncertainty dynamics in an open, interactive decision process, and a new framework is necessary for this. To initiate a formal discussion, we first \textbf{provide a concrete definition and a general formulation of agent UQ} that captures broad categories of existing UQ setups (Sec.~\ref{sec:formulation}). 
Specifically, we model the agent's problem-solving trajectory as a stochastic process over
actions, observations, and environment states, represented as a simple graphical model. Then, we define both turn-level and trajectory-level uncertainty in a simple, probabilistic expression, where we show that many prior UQ problems can be cast as special cases. 

Based on this foundation, we \textbf{identify four technical challenges in agent UQ} that call for actions from communities. (1) Selection of uncertainty estimator: while it is already known that existing uncertainty estimators have their own tradeoffs, agentic scaffolding amplifies these into severe limitations or hard constraints; (2) Uncertainty of heterogeneous entities: how can we derive estimates over the content generated by external entities having different underlying distributions?; (3) Modeling uncertainty dynamics in interactive systems: traditional weighted-average-style uncertainty aggregation may not be suitable to modeling uncertainty dynamics in the non-isolated, open environment; (4) Lack of fine-grained benchmarks: our pilot survey reveals that granular benchmarks, such as turn-level evaluation, are scarce for LLM agents.

Then, Sec.~\ref{sec:impli} discusses \textbf{practical implications} of agent UQ across diverse domains, including healthcare, software engineering, and robotics, as well as frontier LLM research on adaptive reasoning and post-training. Finally, we outline \textbf{open problems} specific to uncertainty modeling in agentic systems (Sec.~\ref{sec:open_probs}). Together, we view this work as a step toward providing the field with a necessary foundation and clearer direction for future research on uncertainty-aware, reliable agentic systems.
\begin{itemize}
\item We cast agent UQ as a distinct problem setup, and build a foundation for that by providing a concrete definition and a general formulation that expresses many existing UQ setups as special cases.
\item We crystallize the emerging challenges in agent UQ with numerical studies on a real-world benchmark $\tau^{2}$-bench~\citep{barres2025tau} with GPT-4.1~\citep{openai2025gpt41} and Kimi-K2.5~\citep{kimi2026}.
\item To further give actionable insights and promote discussion, we describe promising examples of agent UQ applications in various domains and also remaining open problems.
\end{itemize}
\section{Related Work} \label{sec:literature}

\paragraph{UQ in LLMs.} In contrast to classic UQ in machine learning~\citep{neal1992bayesian,mackay1992practical,gal2016dropout}, LLM UQ brings extra challenges from its computational cost and free-form outputs. Common approaches include aggregating output token probability~\citep{zhang2025token,marina2020unsupervised,malinin2021uncertainty,zhang2023enhancing,fadeeva2024fact,duan2024shifting,aichberger2024rethinking}, measuring consistency over multiple generations~\citep{manakul2023selfcheckgpt,farquhar2024detecting,nikitin2024kernel}, verbalizing confidence~\citep{kadavath2022language,lin2022teaching,yona2024can}, or conformal prediction~\citep{kumar2023conformal,cherian2024large,quach2024conformal}. While promising, they mostly use single-step question-answering as a testbed, which limits their practical applicability to agentic setups.

\begin{figure*}
    \centering
    \vspace{-0.5em}
    \includegraphics[width=\linewidth]{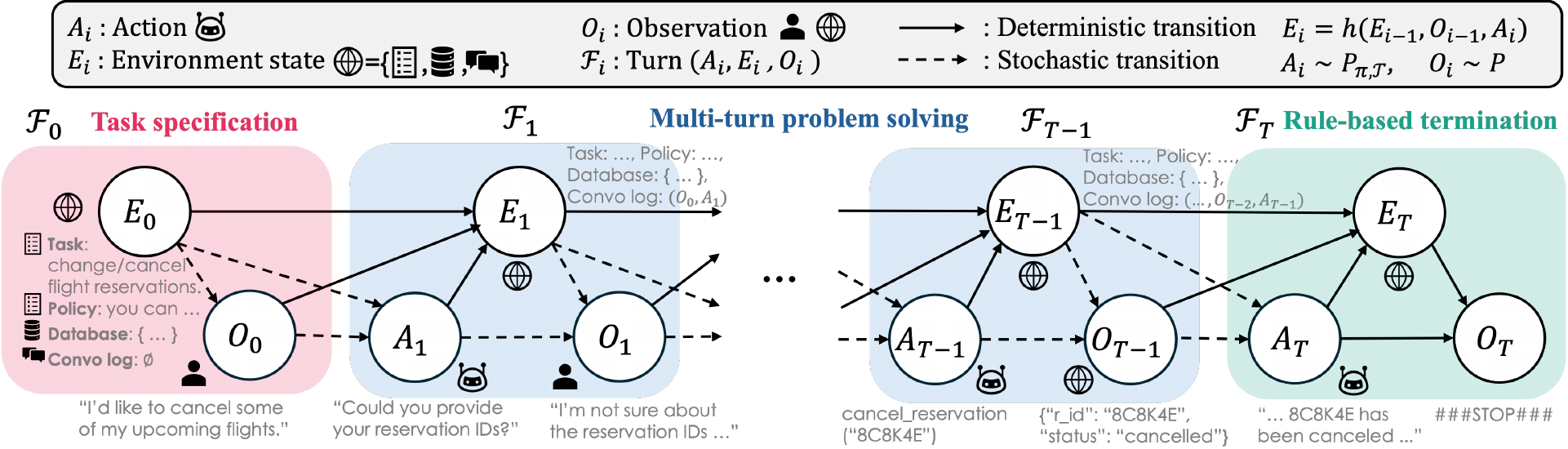}
    \caption{\textbf{Graphical model for an agent problem-solving trajectory with examples.} Given a task specification $E_0$ and an initial user query $O_0$, an agent spans a multi-turn trajectory characterized by a chain of action $A$, observation $O$, and environment state $E$. This simple abstraction describes some representative agentic prompting methods such as ReAct~\citep{yao2022react}. See Appendix~\ref{sec:apdx:formulation:prompting} for details.}
    \vspace{-0.5em}
    \label{fig:pgm}
\end{figure*}
\paragraph{UQ in reasoning and agent setups.} Beyond single-step generation, researchers show growing interest in measuring the confidence of LLMs during their reasoning. This (un)certainty information has been used to improve response quality~\citep{wang2023selfconsistency,kang2025scalable,fu2025deep} or to guide a model during inference or training~\citep {hu2024uncertainty,lightman2024lets,zhu2025uncertainty,wang2025beyond,wang2025harnessing,cheng2025reasoning}.  Recently, there have been attempts to model uncertainty in interactive or agentic inference~\citep{han2024towards,zhao2025uncertainty,duan2025uprop,frankel2024conformal,chan2025conformal,zhang2026confidence,zhang2026auq}. Although aiming to model uncertainty dynamics over long-horizon generation, they \textit{do not reflect uncertainty from different entities} (e.g., users) \textit{nor concern the reducible nature of uncertainty in open environments}~\citep{kirchhof2025position,suri2025structured}, highlighting the need for a new UQ framework under agentic scaffolding. 

\section{A General Formulation for Agent UQ} \label{sec:formulation}

\subsection{Problem Statement and Definitions} \label{sec:formulation:def}
Imagine a flight booking agent depicted in Figure~\ref{fig:illu_traj} right. 
Here, the types of \textit{\textbf{action}} $A$ and \textit{\textbf{observation}} $O$ are diverse: for $A$, there is information-seeking, questioning for clarification, calling an API function, and thinking or updating without interactive behaviors. The observation $O$, on the other hand, can be a database status report, tool-calling results, or the user's clarification. Both action and observation can be realized in flexible forms of natural language messages or structured strings, e.g., programming language. 

Concretely, given a \textbf{\emph{task specification}} $E_0$, e.g., a flight booking, a user makes a stochastic \textbf{\emph{initial query}} $O_0$, and the agent first requests the user's ID to get relevant information. Then, it calls a programmed function to get privileged information from the database. Not only does the agent interact with the user or tool, but it also reasons itself about better future actions (i.e., planning) or makes a system update without external observation. All previous interactions are stored in \textbf{\emph{environment state}} $E$ for context-grounded actions over a long horizon inference. This makes up a \textbf{\textit{trajectory}} $\mathcal{F}_{\leq t}$ (chain of $A$, $E$, and $O$) up to $t$-turn which continues until the task-specific termination rule is met, e.g., when the agent submits an answer~\citep{yang2018hotpotqa,jimenez2024swebench,wei2025browsecomp}, the user generates a stop sign~\citep{yao2024tau,barres2025tau}, or goal completion~\citep{shridhar2021alfworld} with multiple constraints~\citep{xie2024travelplanner}.

Now, to formalize UQ in this real-world agent context,
we begin with a concrete definition of the agent system in Definition~\ref{def:agent}, and introduce a graphical model for the agent trajectory in Figure~\ref{fig:pgm}. This can generalize many existing agent benchmark setups~\citep{mialon2023gaia,yao2024tau,yang2018hotpotqa,song2023llm,jimenez2024swebench,xie2024travelplanner,liu2024agentbench,wei2025browsecomp}.
\begin{tcolorbox}[width=\linewidth, colback=white!95!black, left=1mm, right=1mm, top=1mm, bottom=1mm]
\begin{definition}[Stochastic Agent System] \label{def:agent}
    Let $E_i$ be an environment state, a mixture of (1) a context memory of agent-user-tool interaction logs up to the $i$-th turn that is fully accessible to the agent, and (2) a state of the system database, which is not directly accessible to the agent---only partially observable via tool-calling. Let $O_i$ and $A_i$ be the $i$-th turn observation and action derived from distributions $P$ and $P_{\pi,\mathcal{T}}$, respectively, where $\mathcal{T}$ indicates a tool set and $\pi$ is an LLM policy. Then, under a task specification $E_0$ and the user's initial query $O_0$, we define the trajectory as a sequence of turns $\mathcal{F}_{\leq T} = \{ (A_t, E_t, O_t) \}_{t=0}^T$ where $\mathcal{F}_{0}=(E_{0},O_0)$. Given a deterministic update function $h(\cdot)$, the generative process for each turn $t$ is defined as:
    \vspace{-0.4em}
    \begin{align}
        A_i&\sim P_{\pi,\mathcal{T}}(\cdot|E_{i-1},O_{i-1}),~~O_i\sim P(\cdot|A_i,E_i),~~E_i=h(E_{i-1},O_{i-1},A_{i}).\nonumber
    \end{align}
\end{definition}
\end{tcolorbox}
To express this trajectory with a simple mathematical representation, we introduce a dynamic Bayesian network~\citep{murphy2002dynamic} in Figure~\ref{fig:pgm}. With this graphical model, dependency structures of $A$, $E$, and $O$ can be explained as follows: (1) since the environment state $E_{i-1}$ contains the entire conversation history of $A_{\leq i-1}$ and $O_{<i-1}$ up to the $i-1$ step\footnote{An adaptive memory structure~\citep{xu2025mem,xiong2025memory} can replace this ever-growing memory.}, the current action $A_i$ depends only on $E_{i-1}$ and $O_{i-1}$; (2) the current observation $O_i$ from the environment also depends only on $A_i$ and $E_i$; (3) the transition of the environment state ($E_{i-1}\rightarrow E_i$) can be defined as a deterministic function\footnote{In general, this system transition can be also stochastic, but we consider a simplified yet still realistic scenario.} $h(\cdot)$ which stacks the conversation logs up in a memory module and conducts a rule-based modification (write, update, or delete) on a previous system database state given the current action.

Under this Bayesian network, the joint probability for the trajectory can be factorized as below,
\begin{align}
    P(\mathcal{F}_{\leq T}) &= P(E_0,O_0)\Pi_{i=1}^{T} P(\mathcal{F}_{i}|\mathcal{F}_{i-1}) \nonumber \\
    &= P(E_0,O_0)\Pi_{i=1}^{T} P_{\pi,\mathcal{T}}(A_{i}|E_{i-1},O_{i-1})P(O_{i}|A_{i},E_{i}), \nonumber
\end{align}
where $\mathcal{F}_{0}=(E_0,O_0)$, and $P(E_i|\cdot)=1$ as it is determistic. Now, we define the UQ problem for this stochastic agent system as follows.

\begin{tcolorbox}[width=\linewidth, colback=white!95!black, left=1mm, right=1mm, top=1mm, bottom=1mm]
\begin{definition}[Agent UQ]\label{def:auq}
    Let $U(X)\geq0$ be an uncertainty function that takes a random variable $X$ or its distribution to produce a non-negative real value. Given a $T$-turn trajectory, agent UQ aims to estimate both a turn-level uncertainty $U(\mathcal{F}_t|\mathcal{F}_{t-1})$ for $t=1,...,T$ and a trajectory-level uncertainty $U(\mathcal{F}_{\leq T})=U(\mathcal{F}_0,...,\mathcal{F}_{T})$, as a joint total uncertainty. 
\end{definition}
\end{tcolorbox}
In the machine learning field, the most popular choices of the uncertainty function $U(X)$ are Shannon entropy~\citep{shannon1948mathematical} $\mathbb{E}[-\log P(X)]$, negative log probability $-\log P(X=x)$, and their variants. These information-theoretic measures of uncertainty have a nice property, \textit{chain-rule}~\citep{cover1999elements}, making them suitable for sequential modeling. For example, given $U(\mathcal{F}_{t}):=\mathbb{E}[-\log P(\mathcal{F}_{t})]$, we have the following uncertainty expansion, allowing us to model the agent's total uncertainty as a simple arithmetic of component-level uncertainties
\footnote{See Appendix~\ref{sec:apdx:formulation} and~\ref{sec:apdx:uncertainty} for a derivation and alternative measures, e.g., informational energy~\citep{pardo1991information} and nonextensive entropy~\citep{gell2004nonextensive}.}: 
\begin{align}
    U(\mathcal{F}_{\leq T})&=U(E_0,O_0)+\sum_{i=1}^{T} U(\mathcal{F}_{i}|\mathcal{F}_{i-1}) \nonumber \\
    &= U(E_0,O_0)+\sum_{i=1}^{T} [U(A_{i}|E_{i-1},O_{i-1})+U(O_{i}|A_{i},E_{i})],
\end{align}\label{eq:agent_uq_form}
where $U(E_0,O_0)=U(O_0|E_0)+U(E_0)$ is an initial query uncertainty plus a task specification volatility. 
\paragraph{Desideratum.} 
While previous LLM UQ setups usually focus on conditional uncertainty of a final answer, agent UQ aims at estimating \textit{joint} uncertainty over the full trajectory, where the estimate from a calibrated agent should be predictive of the reward on that trajectory. Formally, for all $r_1>r_2$,
\begin{equation}
    \mathbb{E}\big[U(\mathcal{F}_{\leq T})|r(\mathcal{F}_{\leq T})=r_1\big] < \mathbb{E}\big[U(\mathcal{F}_{\leq T})|r(\mathcal{F}_{\leq T})=r_2\big], \label{eq:agentuq_goal}
\end{equation}
\noindent where $r(\cdot)$ denotes a real-value reward function commonly set as a success-failure binary verifier.

We draw analogies between agent UQ and probabilistic Turing machines~\citep{gill1974computational}, as well as belief tracking in Partially Observable Markov Decision Processes~\citep{kaelbling1998planning}, highlighting that agent UQ is grounded by, while still having its unique edge in contrast to these existing theoretical models in terms of setups and focus. See Appendix~\ref{sec:apdx:literature} for the relevant discussion.

\subsection{A Unified View on Existing UQ Setups} \label{sec:formulation:unif}
We now show that our agent UQ formulation captures various existing UQs as special cases.
\paragraph{Reduction to single-step LLM UQ.} When $t=1$ and the action space is restricted to responses, LLM UQ can be cast as a special case
\begin{equation}
    U(\mathcal{F}_{\leq T})=U(O_0)+U(A_1|O_0) \geq U(A_1|O_0), \nonumber 
\end{equation}
where $O_0$ and $A_{1}$ denote the given question and the corresponding model answer, respectively. Most current LLM UQ literature addresses only the conditional answer uncertainty $U(A_1|O_0)$~\citep{malinin2021uncertainty,fadeeva2023lm,farquhar2024detecting}. A few also consider the question ambiguity $U(O_0)$~\citep{hou24decomposing,tomov2025the} to reflect the reality of wild user query~\citep{min2020ambigqa}. 
\paragraph{Reduction to multi-step reasoning UQ.} When the action space involves multi-step reasoning (e.g., chain-of-thought, \citet{wei2022chain}), we have
\vspace{-0.65em}
{\small
\begin{align}
    U(\mathcal{F}_{\leq T})&=U(O_0)+\sum_{i=1}^{T} U(A_i|A_{<i},O_{0}) \label{eq:lrmuq-exact} \\
    &\geq \sum_{i=1}^{T} U(A_i|A_{<i},O_{0}) \label{eq:lrmuq-noambig} \\
    &\geq \max_{i\in\{1,...,T\}} U(A_i|A_{<i},O_{0}) \label{eq:lrmuq-max} \\
    &\geq \sum_{i=1}^{T}w_i U(A_i|A_{<i},O_{0}),\label{eq:lrmuq-wsum}
\end{align}
}
where $w_i\geq 0$ and $\sum_{i=1}^{T} w_i=1$. Here, $O_0$ and $A_T$ denote the given initial query and the model's final response, and $A_i,...,A_{T-1}$ are intermediate outputs.
Similar to the single-step setup, most existing multi-step UQs~\citep{tanneru2024quantifying,tao2025revisiting} do not consider the question ambiguity $U(O_0)$, so that reduces Eq.~\ref{eq:lrmuq-exact} to Eq.~\ref{eq:lrmuq-noambig}, while a few consider it~\citep{leang2025picsar}. Meanwhile, rather than summing up all the step-wise uncertainties, some approaches consider the most uncertain one, i.e., the lowest confidence used in \citet{fu2025deep}, which can be represented by Eq.~\ref{eq:lrmuq-max}. Finally, Eq.~\ref{eq:lrmuq-wsum} expresses various weighted average methods, e.g., length-normalized uncertainty~\citep{kang2025scalable} with $w_i=\frac{1}{T}$, and tail confidence~\citep{fu2025deep} by setting $w_T=1$, and other clever weighting strategies~\citep{zhao2025uncertainty}.

\paragraph{Connection to process reward modeling.} 
We further draw a \textit{connection between this multi-turn UQ problem and the process reward modeling problem}~\citep{uesato2022solving,lightman2024lets, li2025pqm} where we assign credit (reward) for each intermediate reasoning step, and then it can be utilized to evaluate the reasoning trajectory by aggregating the rewards across steps to derive a trajectory-level reward. For example, \citet{lightman2024lets} proposed the product of the step-wise probabilities and the minimum probability between steps as the aggregation method. These are analogs to Eq.~\ref{eq:lrmuq-noambig} and Eq.~\ref{eq:lrmuq-max}, respectively.
\section{Technical Challenges in Agent UQ} 
\label{sec:challenges}

We discuss what makes agentic setups more challenging for UQ than non-agentic ones, identifying four distinct challenges. To motivate discussion, we present preliminary results on $\tau^{2}$-bench using existing UQ methods: negative log-likelihood (NLL), entropy over token probabilities (Entropy), and verbalized confidence, each aggregated over all tokens or turns generated by an agent throughout the trajectory (Appendix~\ref{sec:apdx:exp} has details).

\subsection{Selection of Uncertainty Estimator} \label{sec:challenges:measure}
\begin{wraptable}{r}{0.485\textwidth}
\scriptsize
\centering
\vspace{-0.825em}
\begin{tabular}{lccc}
\toprule
\textbf{Approach} & \textbf{Access} & \makecell{\textbf{Inference}\\\textbf{Overhead}} & \makecell{\textbf{Theoretical}\\\textbf{Ground}} \\
\midrule
Probability-based       & \xmark & Free        & \cmark \\ 
Consistency-based       & \cmark & Significant & \cmark \\ 
Verbalized Confidence   & \cmark & Neglectable & \xmark \\ \bottomrule
\end{tabular}
\vspace{-0.3em}
\captionof{table}{\textbf{Comparison between UQ approaches.} Existing UQ approaches have unique weaknesses that become severe obstacles in agentic setups.} 
\label{tab:uncertainty_tradeoff}
\vspace{-0.625em}
\end{wraptable}
So what kind of uncertainty estimator should we adopt to model the agent's uncertainty, $U(\cdot)$ in Eq.~\ref{eq:agent_uq_form}? We can start by examining three representative approaches that have been advanced in LLM UQ literature so far: output probability-based (and/or internal hidden state-based), response consistency-based, and prompting-based verbalized confidence methods. 
These approaches involve clear tradeoffs along three dimensions: accessibility, inference-time cost overhead, and theoretical grounding. Probability-based methods exhibit limited applicability due to their reliance on access to output probabilities. While consistency-based and verbalized confidence methods do not face this accessibility constraint, they incur substantial inference-time overhead from repeated generation and lack theoretical grounding, respectively.
\begin{table}[!t]
\centering
\resizebox{\textwidth}{!}{%
\begin{tabular}{lc|ccc|ccc|ccc}
\toprule
 \textbf{Scenario} & \textbf{Avg. $r$} & \multicolumn{3}{c}{\textbf{NLL}} & \multicolumn{3}{c}{\textbf{Entropy}} & \multicolumn{3}{c}{\textbf{Verb.\ Conf.}} \\
 & & AUROC & $\rho$ & $\tau$ & AUROC & $\rho$ & $\tau$ & AUROC & $\rho$ & $\tau$ \\ \midrule
GPT-4.1 on Retail & 0.509 & \textbf{0.597} & 0.169 & 0.138 & 0.580 & 0.139 & 0.114  & 0.575 & \textbf{0.179} & \textbf{0.170} \\
GPT-4.1 on Telecom & 0.517 & 0.624 & 0.214$^{*}$ & 0.176$^{*}$ & 0.611 & 0.192$^{*}$ & 0.158$^{*}$ & \textbf{0.685} & \textbf{0.330}$^{***}$ & \textbf{0.286}$^{***}$ \\
Kimi-K2.5 on Retail & 0.447 & 0.469 & -0.054 & -0.044 & 0.468 & -0.056 & -0.045 & \textbf{0.523} & 0.039 & 0.032 \\
Kimi-K2.5 on Telecom & 0.965 & 0.645 & 0.093 & 0.076 & \textbf{0.664} & \textbf{0.104} & \textbf{0.086} & 0.580 & 0.051 & 0.042 \\ \bottomrule
\vspace{-0.9em}
\caption{\textbf{Evaluation of uncertainty estimators on $\tau^{2}$-bench}. NLL and Entropy predict task failure ($1 - r$); Verbalized confidence predicts task success ($r$), where $r$ denotes the reward; $\rho$ and $\tau$ denote Spearman and Kendall rank correlation coefficient, respectively. The significance is noted as $^{*}p<0.05$ and $^{***}p<0.001$.}
\label{tab:uq_eval}
\end{tabular}}
\vspace{-1em}
\end{table}

Although this is well-established discussion in the classical single-turn QA setup, \textbf{agentic scaffold turns this into serious challenges}: (1) probability-based methods cannot be applied to most frontier LLMs that centered at challenging agent benchmarks, as they do not consistently provide probability outputs; moreover, free-form long generation makes aggregated token probabilities less informative; (2) consistency-based methods become infeasible due to their prohibitively high inference cost in long-horizon, multi-turn settings; and (3) dynamically expanding context memory of agent (i.e., $E_i$ in Eq.~\ref{eq:agent_uq_form}), often coupled with noisy observations, results in increasingly inflated and unreliable verbalized confidence. A pilot study in Table~\ref{tab:uq_eval} provides some results of probability-based and verbalized confidence methods for estimating the trajectory uncertainty $U(\mathcal{F}_{\leq T})$ with aggregated action uncertainty $\frac{1}{T}\sum_{i=1}^{T} U(A_{i})$. In many cases, the considered methods show close-to-random classifier performance in predicting the agent's success/failure. We believe that establishing a theoretically grounded verbalized confidence method is a promising future direction, given its vivid merits in agentic setups such as accessibility and minimal extra inference cost.

\subsection{Uncertainty of Heterogeneous Entities} \label{sec:challenges:obs-uq}
\begin{wraptable}{r}{0.525\textwidth}
    \centering
    \vspace{-0.95em}
    \includegraphics[width=\linewidth]{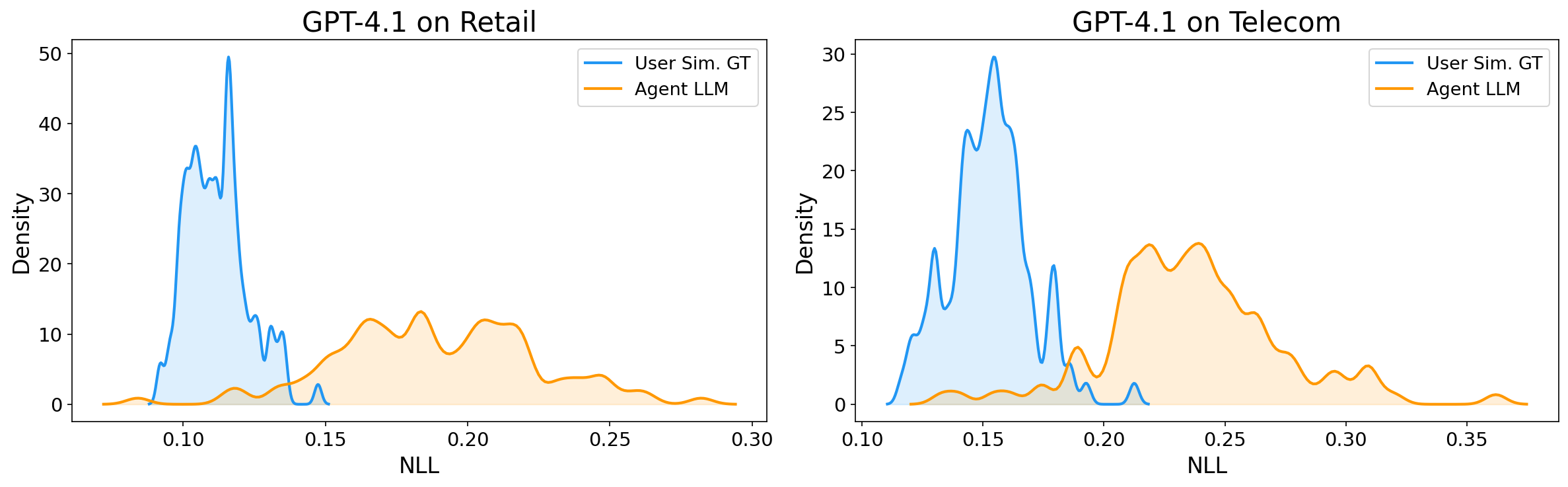}
    \vspace{-1.1em}
    \captionof{figure}{\textbf{Observation uncertainty on $\tau^{2}$-bench.} We compare distributions over observation uncertainty estimates, i.e., average NLL over all user messages, from the agent LLM and the ground truth user simulator LLM.}
    \label{fig:obs_uctt}
    \vspace{-0.85em}
\end{wraptable}
Inference of agents involves not only the actions $A_i$ from the agent itself, but more importantly, observations $O_i$ through the conversations with the user and interactions with the environment via tools. There are \textit{no trivial means to estimate the uncertainty of observations that come from these heterogeneous external entities}, departing from the epistemic uncertainty of the agent-self. 

One may argue that we can just feed the whole trajectory, including the observation strings from the user and tools, into the agent and then use its output probability $P_{\pi,\mathcal{T}}(O_i|A_i,E_i)$ to approximate the underlying true observation distribution $P(O_i|A_i,E_i)$ per turn and corresponding uncertainty $U(O_i|A_i,E_i)$. However, the distribution of human language is definitely different from that of LLMs~\citep{munozortiz2024contrasting,alex2025do}, and the distribution over the tool-calling outcome strings might be different too. To quantify this gap, we conduct a preliminary experiment on $\tau^{2}$-bench, where the user is simulated with another LLM equipped with a distinct prompt from that of the agent LLM. In Figure~\ref{fig:obs_uctt}, we compare the trajectory-level average NLL density plots driven from the ground truth user simulator and the agent LLM. We see remarkable deviations between them, stressing the need for deliberation to estimate observation uncertainty. Adopting auxiliary LLM as an approximated world model can be a possible solution, where we see some promises in Figure~\ref{fig:obs_uctt_full}.

\subsection{Modeling Uncertainty Dynamics in Interactive Systems} \label{sec:challenges:dynamics}
One of the most distinctive features of agent UQ is that the model now operates in a non-isolated environment to generate a trajectory from iterative interactions. This presents a unique challenge: modeling multi-turn cascade uncertainty in an interactive system. Here, the desired UQ method should not only model the propagation of uncertainty across the trajectory but also \emph{account for additional information gained through interactions}. To be specific, an agent in this open environment can gather extra information, make clarifications, and request confirmations at each step by interacting with the user and databases through conversations and tool-calling. These interactive actions are the driving force of goal achievement in the complex (and usually under-specified) agentic tasks in the wild~\citep{li2026thinking}. Therefore, the uncertainty model, which can reliably inform the agent's failures in this setting, should reflect \textbf{conditional uncertainty reduction through information-seeking behavior as well as the uncertainty propagation}. 

\begin{wraptable}{r}{0.525\textwidth}
    \centering
    \vspace{-0.7em}
    \includegraphics[width=\linewidth]{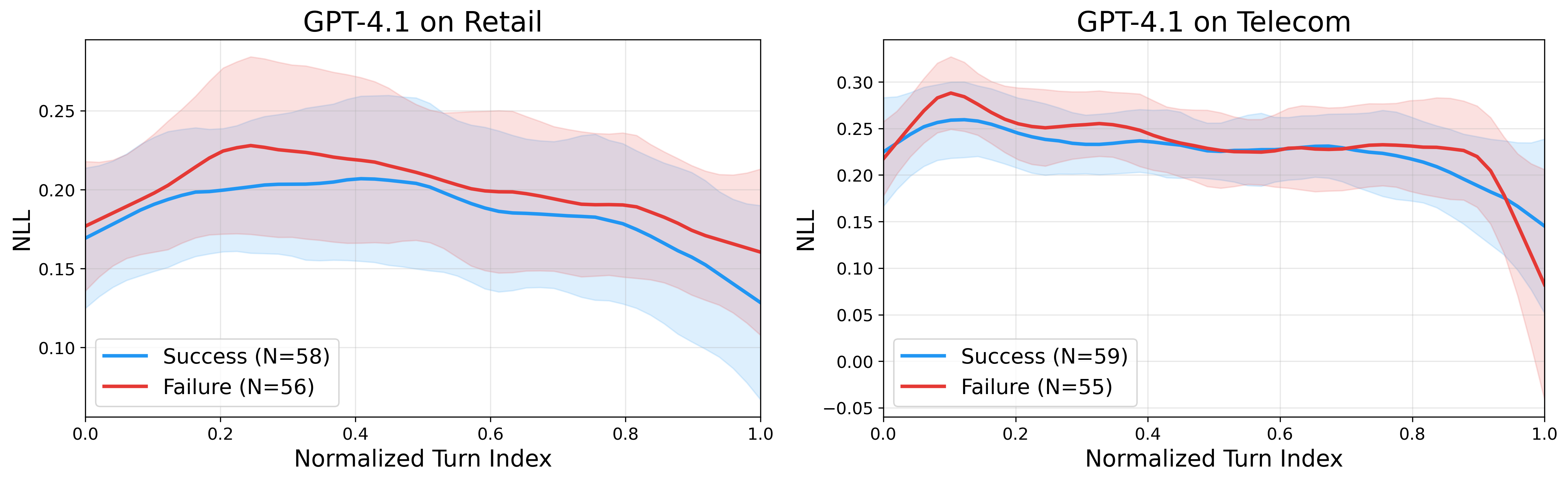}
    \vspace{-1.45em}
    \captionof{figure}{\textbf{Uncertainty evolution across the trajectory on $\tau^{2}$-bench.} We visualize the evolution of action uncertainty estimates across the normalized turn index, averaged over the entire trajectories from two groups: task success and failure.}
    \label{fig:uncertainty_evo}
    \vspace{-0.65em}
\end{wraptable}
Unfortunately, most existing multi-step UQ approaches discussed from Eq.~\ref{eq:lrmuq-exact} to Eq.~\ref{eq:lrmuq-wsum} do not consider the interactivity of actions for each turn, and only focus on uncertainty propagation, while neglecting the reducible nature of uncertainty dynamics in non-isolated, interactive systems. In Figure~\ref{fig:uncertainty_evo}, we present the uncertainty evolution of a naive averaging-based uncertainty aggregation method (Eq.~\ref{eq:lrmuq-wsum}), which is averaged over all trajectories from both success and failure groups. This simple, action-independent aggregation fails to meaningfully discern the failure group from the success group; Even the failure group achieves a sharper uncertainty decrease in the later part of trajectories in Telecom domain. This implies that the naive cascade modeling cannot robustly identify the agent's failure both in the process and at the termination of the task, indicating the necessity of a new framework, pursuing interactivity-centric action-conditional modeling. We discuss a possible recommendation on this action-conditional uncertainty dynamics model in Appendix~\ref{sec:apdx:action-conditional}.

\subsection{Lack of Fine-grained Benchmarks} \label{sec:challenges:bench}
  
\begin{minipage}[h]{0.59\textwidth}
In the meantime, similar to the process reward modeling~\citep{choudhury2025process}, annotating long-horizon agent trajectories incurs non-trivial effort and cost. The absence of fine-grained benchmarks becomes a natural bottleneck to developing a solid UQ method for agents. In this subsection, we provide a mini-survey of existing agent or multi-turn interaction benchmarks (See Appendix~\ref{sec:apdx:survey}), focusing primarily on LLM-centric agents rather than multimodal or physical ones. Specifically, we categorize 44 benchmarks into three classes depending on evaluation granularity: trajectory-level (conducted once at the end of the trajectory), milestone-level (involving several intermediate milestones or events), and turn-level (conducted at every single turn). Figure~\ref{fig:agentbench_survey} reveals that only 4 out of 44 benchmarks provide turn-level annotation, which is very scarce compared to the trajectory-level benchmarks. Besides, although the milestone-level alternative shows better population than the turn-level one, it is still lacking, resulting in three times less than the trajectory-level one.
\end{minipage}%
\hfill
\begin{minipage}[h]{0.385\textwidth}
    \centering
    \includegraphics[width=\linewidth]{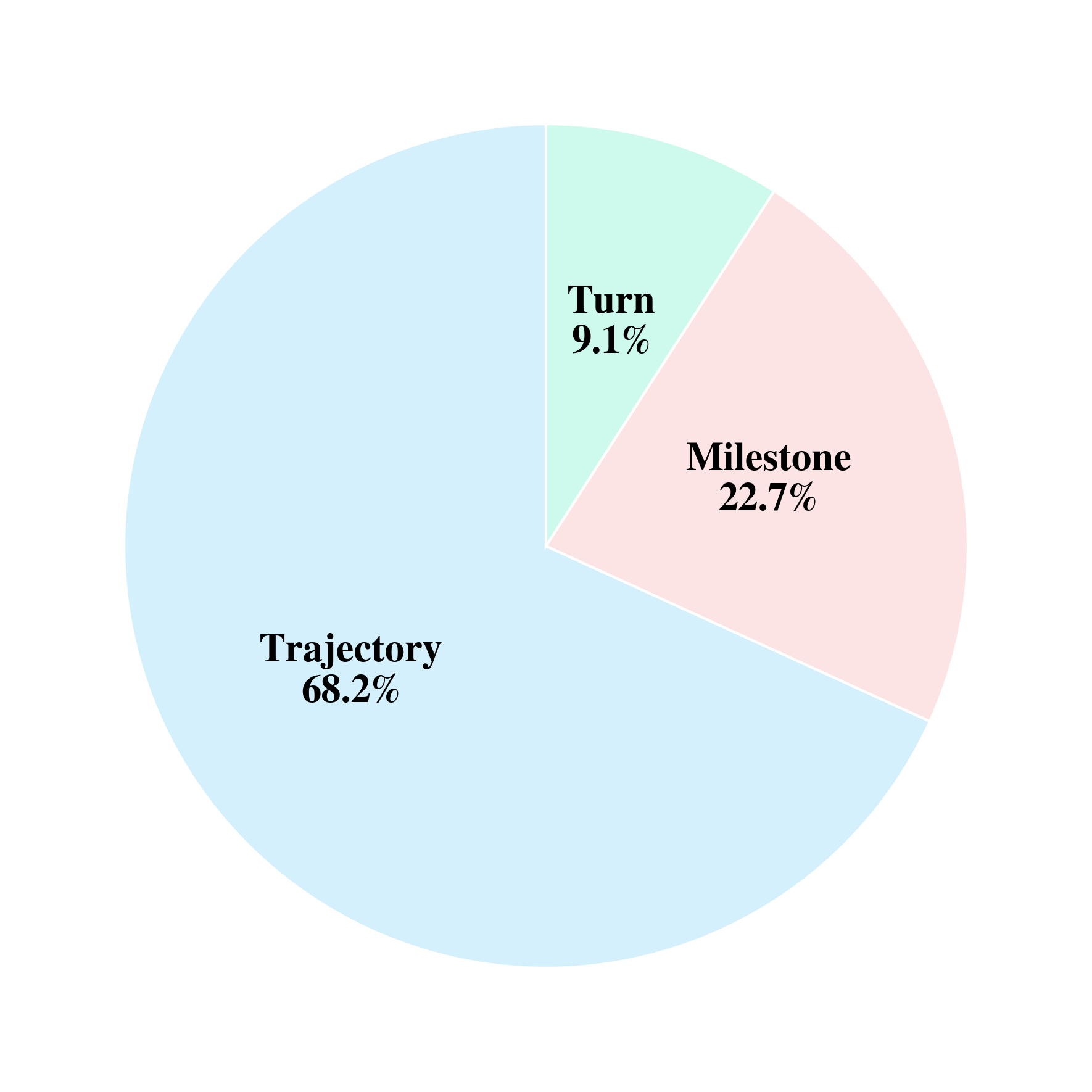}
    \vspace{-1.6em}
    \captionof{figure}{\textbf{Survey on 44 modern agent benchmarks}. The number denotes the percentage of papers in each category.}
    \label{fig:agentbench_survey}
\end{minipage}%

There are several obstacles to developing fine-grained benchmarks for agentic inference setups. First, agent benchmarks often involve very long rollouts that can easily exceed a hundred steps, making it hard to manually annotate each step. Although using frontier models as a substitute for human labeling has become increasingly common, running such models on long-horizon tasks can be prohibitively expensive for many researchers. Moreover, agent benchmarks are usually defined by complex policies; even for humans, it may take a couple of hours to fully understand the task configuration. To accelerate the development of agent UQ, we should pay greater attention to scalable dataset construction strategies for agentic inference.
\section{Practical Implications} \label{sec:impli}
Developing the agent UQ framework is not merely a theoretical exercise but a prerequisite for deploying LLM agents in non-deterministic real environments. We outline implications for \textit{frontier LLM research} and three specialized domains: \textit{healthcare}, \textit{programming}, and \textit{robotics}, in which agent UQ may have profound downstream effects, thereby incentivizing policymakers, practitioners, and researchers. 

\subsection{Advancing Frontier LLMs} \label{sec:impli:frontier}
Test-time scaling through reasoning is now a de facto standard technique to enhance the quality of LLM responses on challenging tasks~\citep{wei2022chain, muennighoff2025s1}. However, there is still huge room for improvement, e.g., mitigating overthinking and/or pursuing efficiency~\citep{sui2025stop,aggarwal2025optimalthinkingbench}. Agent UQ provides concrete guidance to develop an adaptive reasoning method, i.e., early stopping or interaction invoking, via uncertainty budgeting. In addition, multi-turn reinforcement learning has emerged as a promising yet challenging research theme in eliciting the agentic behavior and reasoning of LLMs~\citep{zhou2024archer,shani2024multi,zeng2025reinforcing}. Agent UQ helps tackle the main bottlenecks, credit assignment and when-to-explore problems~\citep{thrun1992efficient,pecka2014safe,choudhury2025process}, limited positive feedback~\citep{lee2025banel,lee2026semi}, and inspires an information pursuit policy~\citep{chattopadhyay2022interpretable} by leveraging directional and quantitative uncertainty information as feedback.

\subsection{Reliable Clinical Decision Supporting Agents} \label{sec:impli:clinic} 
Modern LLM agents offer a promising vision for healthcare, moving beyond static knowledge retrieval to autonomous systems with complex reasoning and tool-execution capabilities~\citep {wang2025survey}. However, the transition from promising prototypes to reliable clinical agents faces critical obstacles stemming from the system's inability to recognize and manage its own limitations. We believe that agent UQ can fill this gap.
For example, although~\citet{ferber2025development} observed a huge potential of tool-calling LLM agents in oncology diagnoses, the remaining error rate underscores the necessity of a ``human-in-the-loop'' workflow. Here, agent UQ can serve as a gatekeeper, e.g., flagging highly accumulated uncertainty moments to invite humans, automating during low-uncertainty periods, while pursuing total uncertainty reduction for the right final diagnosis. This aligns with the concept: \textit{adaptive healthcare} that envisions a paradigm shift from passive prediction to risk-aware active inquiry~\citep{hinostroza2025ai}, to tackle some real clinical settings such as disease stage prediction~\citep{yin2023fractional} or brain aging assessment~\citep{cheng2026structural}.

\subsection{Reliable Software Engineering Agents} \label{sec:impli:swe}
Although advanced coding agents~\citep{anthropic2025claude45systemcard,openai_gpt52codex2025} have pushed performance on standard benchmarks~\citep{jimenez2024swebench} to an impressive degree, the disparity between passing a controlled test and safely modifying a production-level codebase stresses the urgent need for a reliable agent UQ framework. 
Agent UQ enables a more controllable exploration and commitment mechanism, wherein the agentic system knows uncertainty over bug patch candidates and decides whether to gather more evidence, e.g., check more files, commit to a fix, or interact with the user. Besides, \textit{uncertainty-triggered rollback and branching algorithm}, in which an agent treats increasing uncertainty after an action as a signal to revert or branch by mirroring the checkpoint-rollback workflows of human engineers, can lead to reliable agentic coders.

\subsection{Reliable Embodied Agents in Cyber-Physical Systems} \label{sec:impli:robot}
Embodied agents in cyber-physical systems face heterogeneous uncertainty from sensing, dynamics, and human intent~\citep{li2024embodied, fung2025embodied, team2025gemini}, and they actively reduce uncertainty through interaction by re-sensing the environment or interacting with humans before committing to irreversible physical actions~\citep{xu2025compliant}. For example, a robot instructed to retrieve a fragile object may explore or seek clarification rather than execute a grasp under high uncertainty.
This makes embodied autonomy an instance of a conditional uncertainty reduction process, where information-gathering actions reduce uncertainty and state-changing actions increase commitment and risk.
By explicitly quantifying and managing uncertainty dynamics, agent UQ offers \textit{safer action selection}~\citep{zhang2025safevla}, \textit{principled delegation with minimal feedback}~\citep{hagenow2025realm}, and \textit{robust execution}~\citep{romer2025failure}.
\section{Open Problems} \label{sec:open_probs}
We have so far presented a roadmap for agent UQ with a general formulation, identified challenges, and practical implications, but important open problems remain as follows.

\paragraph{Intrinsic solution multiplicity.}  
Is high uncertainty per step due to the agent's lack of knowledge about the right action for that step? Or, is it due to the intrinsic multiplicity of valid actions in that step? We cannot identify which of them is the right source of high uncertainty! Besides, in contrast to the classic question answering, the task bid to an agent is too complex to be perfectly specified at once. Thus, each intermediate step is intrinsically ambiguous and allows multiple valid actions, making it impossible to specify the source of the high uncertainty. A concurrent work points out this multiplicity issue and approaches it through an information-theory-inspired framework~\citep{suri2025structured}. Graph-based parallelism/multiplicity modeling might also be worth exploring as a possible solution here~\citep{wu2025regular,choi2026modex}.

\paragraph{Evaluation beyond task failure.} Evaluation of predictive uncertainty has been conducted by measuring the informativeness of uncertainty estimates when predicting the incorrectness of the final answer. While suitable in single-turn QA settings, it becomes insufficient when we move to agentic LLM setups. In such setups, the sole prediction of eventual task failure collapses a rich, structured problem-solving process into a single scalar judgment, discarding useful information that makes uncertainty more actionable. We envision a new evaluation protocol that integrates dynamic solution multiplicity, task difficulty, and irreducible external ambiguity, as well as task failure, that may redefine uncertainty from a multi-facet view~\citep{guo2025uncertaintyprofiles}. This reframing calls for new metrics that credit not just predictive correlation with outcomes but the informativeness of uncertainty signals in multiple factors in agentic scaffolding, which is yet to be explored.

\paragraph{Uncertainty modeling in multi-agent systems.} Although we confine our attention to a single-agent environment, there has been notable advancement in multi-agent systems (MAS) \citep{Liang2023EncouragingDT,du2024improving,guo2024large,wu2024autogen, choi2025debatevote}. Uncertainty dynamics in this MAS can be leveraged to improve the quality of inter-agent communication~\citep{yoffe2025debunc} and mitigate debate collapse~\citep{tang2026value}, thus worth exploring. Unfortunately, while there are some works that borrow the strength of collective intelligence from MAS to improve UQ quality in a single-turn QA setup~\citep{yang2024confidence,feng2025rethinking}, no existing work in this field attempts to design a UQ method tailored for MAS. Modeling the joint uncertainty dynamics of multiple agents brings extra challenges that may require deliberation and other formal toolkits~\citep{tsallis2004nonextensive,j2016generalising,cheng2021general}.
Nonetheless, our graphical model naturally extends to coupled per-agent trajectories with a shared environment state, offering a foundation for UQ in future multi-agent systems.

\paragraph{Uncertainty modeling in self-improving agents.} In our setup, we assume that the agent only operates within a single episode. However, the agent research community has recently shown an increasing interest in self-evolving/improving agents~\citep{bai2022constitutional,shinn2023reflexion,NEURIPS2023_91edff07,wang2024voyager,jiang2025adaptation} that adapt across multiple episodes. Uncertainty information can also be utilized in this setup for reliable continual adaptation~\citep{zhang2026selaur}. However, the interactive self-evolving environment adds another non-trivial factor that shifts uncertainty dynamics as shown in~\cite{huang2025beyond}. Future exploration in this setup should reflect the non-stationary toolkit, context memory, and agent model parameters across multiple episodes.
\section{Conclusion}
We argue that LLM UQ research should move towards a more realistic setup, i.e., agentic inference, featuring the interactive long-horizon task episode, which requires a paradigm shift from a point-wise estimate to the sequential dynamics model of uncertainty. We laid the foundation for this by providing a concrete definition and general formulation of agent UQ that abstracts a wide variety of existing UQ approaches. With that, we identified four key challenges emerging in UQ on agentic scaffolding with detailed explanations and numerical analyses. We further note the future direction to explore agent UQ from the perspectives of real-world applications and open problems to provoke forward-looking discussions. We hope the three pillars provided in this work help communities design their upcoming works, products, and more.

\section*{Limitations}
Agentic systems in the wild may encounter unreliable (even adversarial; \citet{kang2025trap}) observations and a stochastically evolving environment. However, we do not dive deeper on that settings, and our formulation might not provide a proper abstraction. In that scenario, we may need to explicitly model uncertainty in the agent's evidence stream as well, e.g., trust/ignorance about tool outputs, retrieval, and user-provided facts, by leveraging formalisms, such as subjective logic \citep{jsang2018subjective,cheng2020there}, imprecise probability theory \citep{walley1991statistical}, and formal belief representation~\citep{formalbelief22}, to safely manage memory and act under unreliable observations in stochastic world.

\section*{Ethical Considerations}
This work aims to establish a foundation for agent UQ, which can directly contribute to society, science, and well-being~\citep{han2023recipe,ren2025towards,oh2025flex,hendrycks2025definition}. We believe that the realization of our proposed framework can significantly boost the reliability of the LLM agent in a dynamic, interactive inference setup, thereby facilitating risk mitigation in high-stakes decision-making. However, we also acknowledge that the proposed framework can be misused by allowing adversaries to efficiently steer the agent towards its blind spots or manipulate uncertainty thresholds to bypass safety filters~\citep{wang2025agentvigil}. Therefore, we recommend alleviations such as limiting exposure to attack-enabling signals, red-teaming, and monitoring for abuse, which avoid overpromising what low uncertainty implies.

\clearpage

\section*{Acknowledgments}
We sincerely thank Artem Shelmanov, Shawn Im, Hyeong Kyu Choi, Jongwon Jeong, Eunsu Kim, Mingyu Kim, Ayoung Lee, JungEun Kim, Hayun Lee, Hoyoon Byun, and Kyungwoo Song for their sharp feedback on the draft.
This work is supported in part by the AFOSR Young Investigator Program under
award number FA9550-23-1-0184, National Science Foundation under awards IIS-2237037 and IIS-2331669,
Office of Naval Research under grant number N00014-23- 1-2643, Schmidt Sciences Foundation, Open
Philanthropy, Alfred P. Sloan Fellowship, and gifts from Google and Amazon. Paul Bogdan acknowledges the support by the National Science Foundation (NSF) under the NSF Award 2243104 under the Center for Complex Particle Systems (COMPASS), NSF Mid-Career Advancement Award BCS-2527046, U.S. Army Research Office (ARO) under Grant No. W911NF-23-1-0111, Defense Advanced Research Projects Agency (DARPA) Young Faculty Award and DARPA Director Fellowship Award, Okawa foundation award,  National Institute of Health (NIH)
under R01 AG 079957, and Intel faculty awards.

\bibliography{main}
\clearpage
\clearpage

\appendix
\begin{center}
    \LARGE \textbf{Appendix}
    \vspace{1em}
\end{center}
\tableofcontents
\addtocontents{toc}{\protect\setcounter{tocdepth}{2}}

\section{Revision History} \label{sec:apdx:revision}
A prototype version of this paper titled "\textit{Towards Reducible Uncertainty Modeling for Reliable Large Language Model Agents}" was released on arXiv in early February 2026. It was specifically targeted at the third challenge described in the current Section~\ref{sec:challenges:dynamics} by proposing a conceptual framework, conditional uncertainty reduction process, to enable action-conditional uncertainty dynamics modeling for interactive agents, which is available now at Section~\ref{sec:apdx:action-conditional} in this Appendix. After receiving some early feedback from the community, we revised this article to cover a broad range of practical challenges emerging under agentic scaffolding, supported by quantitative analyses, rather than focusing narrowly on a single challenge and its corresponding solution. In the main body of the paper, Section~\ref{sec:challenges} has undergone a major revision, while the remaining sections were subject to only moderate or minor edits.

\section{Additional Details on Formulation} \label{sec:apdx:formulation}
\subsection{Expression of Various Agentic Prompting under Our Formulation} \label{sec:apdx:formulation:prompting}
In Section~\ref{sec:formulation:def}, we have introduced a definition of stochastic agent system (Def.~\ref{def:agent}) with a probabilistic graphical model of specific conditional dependency assumptions in Figure~\ref{fig:pgm}. Under this model, we can abstract some representative agentic prompting methods~\citep{ahn2022can,yao2022react,shinn2023reflexion,patil2025berkeley} with a unified language yet distinctive conditions as below.
\begin{enumerate}
    \item \textbf{Vanilla function calling}~\citep{patil2025berkeley}: Same as Definition~\ref{def:agent}. Stochastic process of $(A_t,E_t,O_t)$ without any extra conditions.
    \item \textbf{Few-shot prompting}~\citep{ahn2022can}: Just augmenting initial task-specification variable as $E^{'}_{0}=E_0\cup\{\mathcal{F}^{(i)}\}_{i=1}^{N}$ with some example rollouts $\mathcal{F}^{(i)}=\{(A_t,E_t,O_t)\}_{t=1}^{T^{i}}$, which may implicitly affect the action transition probability, $P_{\pi,\mathcal{T}}(\cdot)$, by conditioning.
    \item \textbf{ReAct}~\citep{yao2022react,wang2024executable}: Adding a constraint on the action sampling as {\small$A_i\sim P_{\pi,\mathcal{T}}(A^{\texttt{thk}}|E_{i-1},O_{i-1})$} if {\small$A_{i-1}\in \mathcal{A}^{\neg\texttt{thk}}$} else {\small$A_i\sim P_{\pi,\mathcal{T}}(A^{\neg\texttt{thk}}|E_{i-1},O_{i-1})$}, where {\small$A^{\texttt{thk}}\in\mathcal{A}^{\texttt{thk}}$} and {\small$A^{\neg\texttt{thk}}\in\mathcal{A}^{\neg\texttt{thk}}$} randon varaibles from the partitions of $\mathcal{A}$ standing for the space of textual reasoning and all others, respectively.
    \item \textbf{Reflexion}~\citep{shinn2023reflexion}: Do multiple trials $\mathcal{F}^{(\leq i)}=\{\mathcal{F}^{(1)},...,\mathcal{F}^{(i)}\}$ until $r^{(i)}:=r(\mathcal{F}^{(i)})=1$ or reach maximum trials, otherwise augmenting the per-trial initial state with an external long-term memory $E^{(i+1)}_0=E^{(i)}_0\cup\{\mathcal{M}(\mathcal{F}^{(j)},r^{(j)})\}_{j\leq i}$, where the long-term memory $\mathcal{M}(\cdot)$ constructs with textual feedback generated from an LLM prompted by the current trial trajectory and a corresponding binary reward. Thus, the action transition probability, $P_{\pi,\mathcal{T}}(\cdot)$, is continually evolved by reflecting feedback from historical runs. 
\end{enumerate}
Overall, this suggests that our agent formulation is general and flexible enough to capture a broad class of representative agentic prompting scenarios.

\subsection{Uncertainty Instantiations and Total Uncertainty Expansion} \label{sec:apdx:formulation:expansion}
After defining the agent UQ in Def~\ref{def:auq}, we presented a trajectory-level total uncertainty that expresses the joint uncertainty across multiple turns in an additive form of the initial query uncertainty, action uncertainty, and observation uncertainty. 

Given the joint probability, {\small$P(\mathcal{F}_{\leq T})= P(E_0,O_0)\Pi_{i=1}^{T} P_{\pi,\mathcal{T}}(A_{i}|E_{i-1},O_{i-1})P(O_{i}|A_{i},E_{i})$}, we enumerate three example instances of the uncertainty measure $U(\cdot)$ that induces the simple additive form of total uncertainty:
\begin{enumerate}
    \item \textbf{Information content} (negative log probability), {\small$U(X=x):=-\log P(X=x)$}, to measure a point-wise surprisal for a given observation.
    \item \textbf{Entropy}, {\small$U(X):=H(X)=\mathbb{E}[-\log P(X)]$}, to measure the expected amount of surprisal.
    \item \textbf{Relative entropy}, {\small$U(X):=\mathbb{D}_{\rm KL}\big(Q(X)||P(X)\big)$} with a pre-defined reference distribution {\small$Q(X)$} and Kullback-Leibler divergence $\mathbb{D}_{\rm KL}$, to measure the expected amount of surprisal given a prior knowledge (if we have any).
\end{enumerate}
For information content, it is trivial to show by just taking the negative logarithm to {\small$P(\mathcal{F}_{\leq T})$}. For the entropy and relative entropy, we have the following chain rule~\citep{cover1999elements}, 
\begin{itemize}
    \item {\small$H(\mathcal{F}_{0},...,\mathcal{F}_{T})=\sum_{i=0}^{T}H(\mathcal{F}_{i}|\mathcal{F}_{i-1},...,\mathcal{F}_{1})$},
    \item {\small$\mathbb{D}_{\rm KL}\big(Q(\mathcal{F}_{0},...,\mathcal{F}_{T})||P(\mathcal{F}_{0},...,\mathcal{F}_{T})\big)=\sum_{i=0}^T \mathbb{D}_{\rm KL}\big(Q(\mathcal{F}_{i}|\mathcal{F}_{i-1},...,\mathcal{F}_{1})||P(\mathcal{F}_{i}|\mathcal{F}_{i-1},...,\mathcal{F}_{1})\big)$},
\end{itemize}
which directly drives our uncertainty expansion for {\small$U(\mathcal{F}_{\leq T})$} in Section~\ref{sec:formulation:def}.

\section{Alternative Uncertainty Measures} \label{sec:apdx:uncertainty}
Although we have mainly focused on the above three information-theoretic uncertainty measures throughout the paper, there are plenty of alternatives one can consider depending on the problem setups. In this section, we examine some possible alternative uncertainty measures, \textit{Rényi entropy}, \textit{Tsallis entropy}, and \textit{informational energy}. These sophisticated measures may be worth investigating for the LLM agents' inference interface, characterized by long-range interactions, evolving memory, and the multifractal property.

\paragraph{Rényi entropy~\citep{renyi1961measures}} has been widely applied in modern machine learning, ecology, biodiversity science, and quantum information, among other fields. Given an order parameter {\small$0<\alpha<\infty$} and {\small$\alpha \neq 1$}, it is defined as
{\small$H_{\alpha}(X)=\frac{1}{1-\alpha}\log (\sum_{x\in\mathcal{X}} P(x)^{\alpha})$} and generalize many other entropies: max-entropy {\small$H_{0}(X)=\log|\mathcal{X}|$}, Shannon entropy {\small$\lim_{\alpha\rightarrow1} H_{\alpha}(X)=-\sum_{x\in\mathcal{X}} P(x)\log P(x)$}, collision entropy {\small$H_{2}(X)=-\log \sum_{x\in\mathcal{X}}P(x)^{2}$}, and min-entropy {\small$H_{\infty}(X)\doteq \min_{x\in\mathcal{X}}-\log P(x)$}.

\paragraph{Tsallis entropy~\citep{tsallis1988possible}} extends this further to model a complex system where the additivity property of Shannon and Rényi entropies for independent subsystems does not hold. Given an index parameter $q\in\mathbb{R}\backslash \{1\}$, it is defined as {\small$H_{q}(X)=\frac{1-\sum_{x\in\mathcal{X}}P(x)^{q}}{q-1}$}, which recover Shannon entropy when $q\rightarrow1$. Along with Rényi entropy, Tsallis entropy has taken a key role in characterizing complex physical systems as a nonextensive measure of entropy~\citep{tsallis2004nonextensive}.

\paragraph{Onicescu informational energy~\citep{pardo1991information,calin2014informational}} is another popular measure of (un)certainty that has bred many applications in quantum mechanics, economics, ecology, social sciences, and so on. It is defined as $\rm{IE}(X):=\sum_{x\in\mathcal{X}}P(x)^{2}$ which can be connected to Rényi entropy of order-2 $H_{2}(X)=-\log \sum_{x\in\mathcal{X}}P(x)^{2}=-\log \rm{IE}(X)$, as well as the power entropy \citep{vajda1968bounds,vajda2007generalized} $V_{2}(X)=1-\sum_{x\in\mathcal{X}}P(x)^{2}=1-\rm{IE(X)}$. There is also a concept, Onicescu’s correlation coefficient~\citep{calin2014informational}, $\rho(X,Y)=\frac{\sum_{x,y} P(x)P(y)}{\sqrt{\rm{IE(X)\cdot\rm{IE(Y)}}}}$, to quantify the dependency structure between multiple variables.

\section{Extended Context} \label{sec:apdx:literature}
\paragraph{Connection to the probabilistic Turing machines.} At some level of abstraction, one can view the LLM agents as probabilistic interactive Turing machines (PITMs)~\citep{john1977turing} that induce a distribution over action-observation transcripts. This connection may allow us to borrow some possible formal tools from the PITM literature to design UQ method for LLM agents. However, we note the following distinctions: (1) PITMs have an explicit randomness source and a defined distribution, whereas LLM agent also has a procedural uncertainty induced by the decoding time strategies; (2) while PITMs interact via fixed channels, LLM agents commonly interact with users and tools in partially observable or changing formats; and (3) the goal of LLM agents are often underspecified upfront and can be negotiated over turns depending on the situation, though PITMs have a concrete accept/output condition.

\paragraph{Connection to the partially observable Markov decision process.} One might draw an analogy between the agent UQ and belief tracking in Partially Observable Markov Decision Processes (POMDPs)~\citep{kaelbling1998planning}. However, the agent theme challenges the standard setups in traditional POMDPs: (1) there is no explicit belief state, but the agent carries an implicit belief in its memory; (2) actions and observations comprise language and structured strings, spanning an effectively infinite space---far from the small discrete spaces in classical POMDP; and (3) it concerns more on the agent's uncertainty, whereas POMDP concerns mostly on environment uncertainty. These connections highlight the distinctive edge of agent UQ, while grounding it in a well-established classical problem set.

\section{Action-Conditional Uncertainty Dynamics Model} \label{sec:apdx:action-conditional} 
\begin{figure*}[t]
    \centering
    \includegraphics[width=\linewidth]{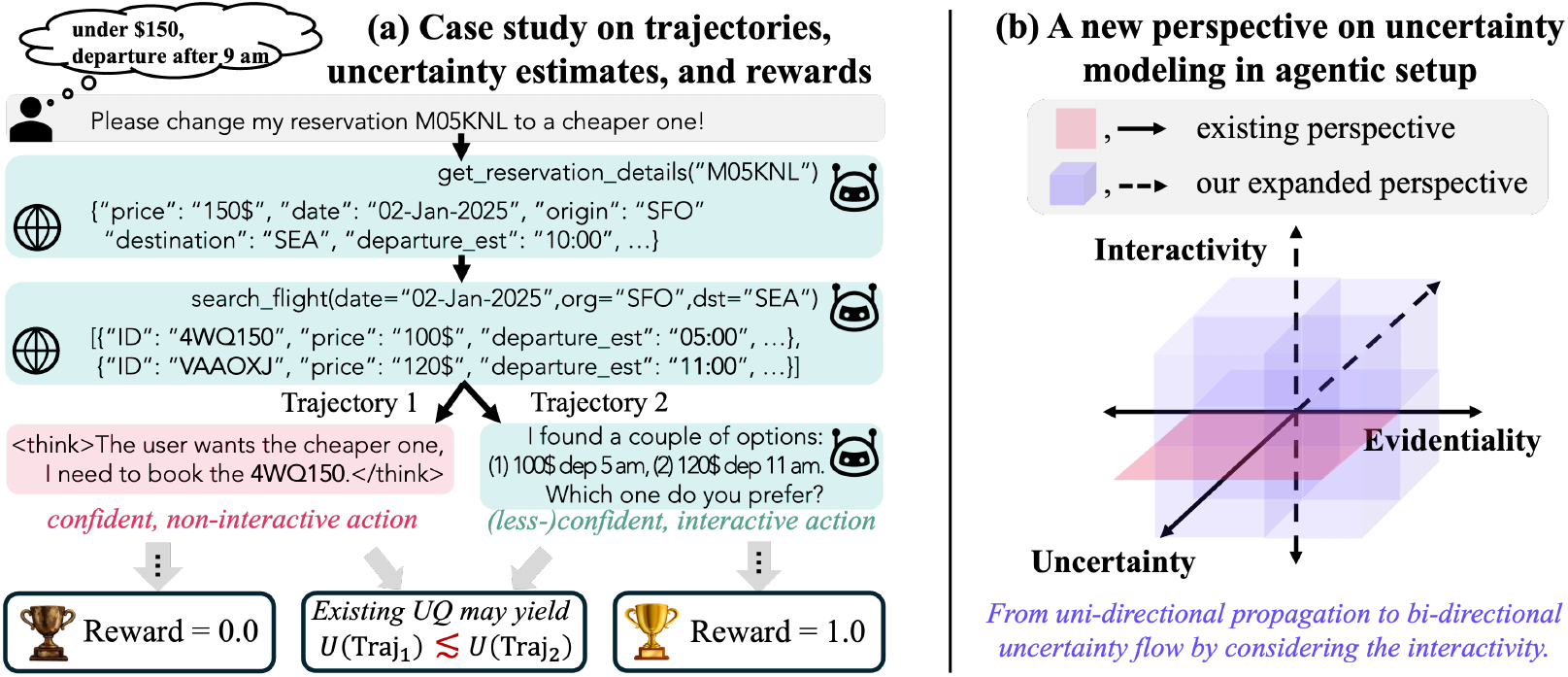}
    \caption{\textbf{Limitation of existing UQ (a) and our suggestion (b).} Prior works just concern the evidentiality when designing or evaluating UQ while neglecting \textit{interactivity}, so they may fail to reliably capture the agent's failure. We recommend expanding a dimension of interest, interactivity, and moving on to tackle the third challenge (Sec,~\ref{sec:challenges:dynamics}) by allowing reducible uncertainty modeling of agents.}
    \label{fig:example}
\end{figure*}
\paragraph{Motivation.} From a unified view discussed in Section~\ref{sec:formulation:unif}, we see that most of the existing UQ methods model the multi-step uncertainty as a uni-directional propagation, i.e., every step-level uncertainty positively contributes to the total uncertainty, without distinguishing action types at each turn. One might claim that the average uncertainty $\frac{1}{T}\sum^{T}_{t=1} U(\mathcal{F}_{t}|\mathcal{F}_{t-1})$ can already be fine, as it might produce a lower value for the chain of certain actions and vice versa. However, it does not consider the type of action which is critical in agentic setups, where an agent's \textit{interactivity} is key to earning a reward. See Fig.~\ref{fig:example} (a), a scenario wherein a user does not fully specify their goal upfront. After searching for options, the agent either reasons itself or interacts with the user to delegate decisions. Although both trajectories are evidential, interaction-oriented ones tend to achieve higher rewards, despite reasoning-oriented ones often yielding greater confidence. Yet, existing methods can produce misleading uncertainty estimates, fail to meet the desiderata (Eq~\ref{eq:agentuq_goal}). From our pilot experiments (Figure~\ref{fig:uncertainty_evo_gpt_full} and Figure~\ref{fig:uncertainty_evo_kimi_full} provided in Appendix~\ref{sec:apdx:exp:result}), we observed that an action-independent naive aggregation strategy could not faithfully separate between the success and failure groups on average, resulting in full overall performance reported in Table~\ref{tab:uq_eval}, regardless of the types of uncertainty estimator.

\paragraph{Recommendation: conditional uncertainty reduction process.} We therefore need a new perspective to model an agent's uncertainty dynamics that accounts for both reductions and increases, explicitly depending on the agent's interactivity at each step. While there are multiple ways to realize this concept, we propose one possible implementation as a prototype: \textit{information gating}, which goes hand in hand with the information-theoretic measure of uncertainty noted in~Sec.~\ref{sec:formulation}. Specifically, we derive a lower bound of total uncertainty $U(\mathcal{F}_{\leq T})$ that admits a selective reduction or increase in uncertainty throughout the trajectory:

{\small
\begin{align}
    U(\mathcal{F}_{\leq{T}})=U(E_0,O_0)+\sum_{i=1}^{T} U(\mathcal{F}_{i}|\mathcal{F}_{i-1}) 
    &= U(E_0,O_0)+\sum_{i=1}^{T} [U(A_{i}|E_{i-1},O_{i-1})+U(O_{i}|A_{i},E_{i})] \nonumber \\
    &= U(E_0,O_0)+\sum_{i=1}^{T} U(A_{i}|E_{i-1},O_{i-1})[1+\frac{U(O_{i}|A_{i},E_{i})}{U(A_{i}|E_{i-1},O_{i-1})}]  \nonumber \\
    &\geq U(E_0,O_0)+\sum_{i=1}^{T} U(A_{i}|E_{i-1},O_{i-1})g(\mathcal{F}_i). \label{eq:infogate}
\end{align} 
\normalsize}
Here, we introduce a conditional gating function $g(\cdot)$ that reduces or increases a current turn uncertainty depending on the action. That is, for a set of valid uncertainty reduction actions, $\mathcal{A}^{-}$, we define:
\begin{center}
$g(\mathcal{F}_i)=\begin{cases}
\frac{- \text{Info}(O_i;O_0|E_{i}\backslash O_0)}{U(A_{i}|E_{i-1},O_{i-1})} & \text{if}~A_i\in 
\mathcal{A}^{-}, \\
1+\frac{U(O_{i}|A_{i},E_{i})}{U(A_{i}|E_{i-1},O_{i-1})} & \text{otherwise.}
\end{cases}$ 
\end{center}
where $\text{Info}(\cdot;\cdot)$ denotes mutual information or pointwise mutual information that measures the \emph{amount of information gain} by having the current observation $O_i$ w.r.t. initial query $O_0$ given conversation history $E_i\backslash O_0$. For each turn, if the agent's action is classified as an interactive (inviting a user or a tool) and evidential one (factual or not in conflict with stored data), then $g(\cdot)$ produces the amount of information gain with negative sign to realize uncertainty reduction; otherwise, it propagates uncertainty (See Fig.~\ref{fig:curp_app}).

\begin{figure*}[!t]
    \centering
    \includegraphics[width=\linewidth]{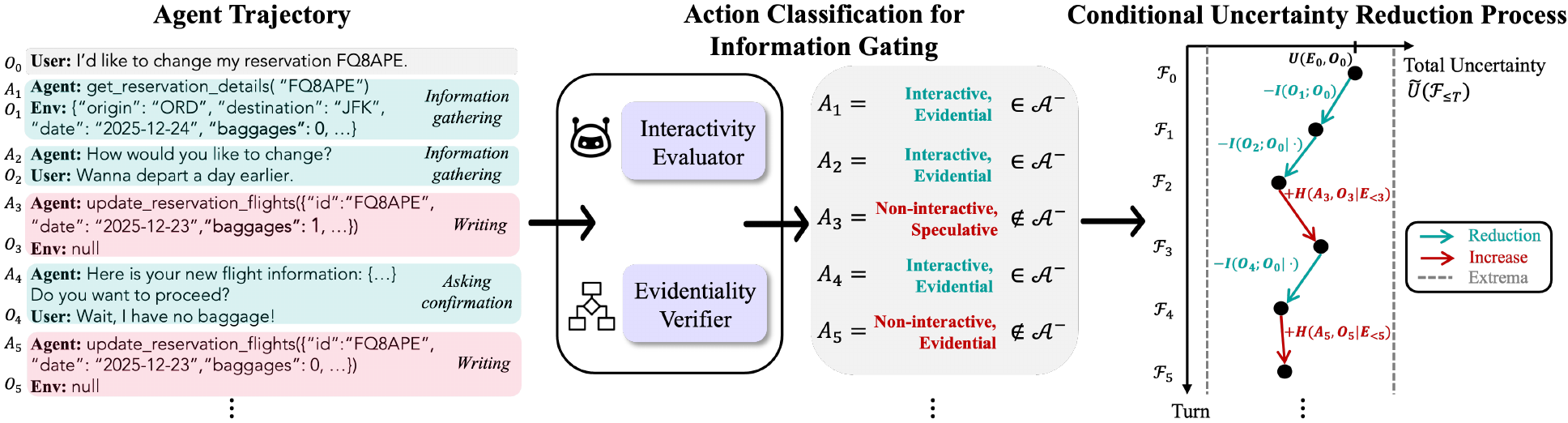}
    \caption{\textbf{Illustration on the proposed agent UQ paradigm.} We propose the conditional uncertainty reduction process for LLM agents by discerning interactive and evidential actions from others.}
    \label{fig:curp_app}
\end{figure*}

Importantly, the signed gating function handles uncertainty in a more interpretable manner than traditional estimates by letting not only quantitative but also \textbf{directional interpretations}. Besides, this uncertainty model allows us to derive the analytic extrema, i.e., maximum and minimum, of total uncertainty in closed forms as shown in Lemma~\ref{lem:apdx:extreme}.

\subsection{Implementation Sketch: Conditional Uncertainty Reduction Process} \label{sec:apdx:action-conditional:impl}
We discuss a possible implementation recipe for this conditional uncertainty reduction process to provide useful guidance for practitioners and researchers.

\paragraph{Action classifier.} To realize the conditional information gate, we first need to classify actions to determine whether each action will contribute to total uncertainty reduction for that task or not. As emphasized in Section~\ref{sec:formulation}, considering the interactivity of actions is key to accounting for agent uncertainty dynamics, as well as the evidentiality (or factuality) of the action that is solely considered for the existing approaches so far. Therefore, we may want to implement a compound classifier with both interactivity assessment and evidentiality verification. They can be implemented through a rule/syntactic-based, simple verifier~\citep{mu2024rule}, an LLM judge with specialized prompting, fine-tuning~\citep{liu2025compassverifier}, or a hybrid of rule-based and LLM-based~\citep{peng2025verif}. Below are the example prompts one can adopt in $\tau^{2}$-bench (Retail) for LLM-judge in interactivity and evidentiality classification.

\begin{tcolorbox}[
  title=Example System Prompt for Interactivity Classification,
  colback=black!5!white,
  colframe=black!60!white,
  fonttitle=\bfseries,
  breakable
]
You are classifying customer service agent actions as either:\\
- "interactive": asking questions, requesting clarification, using tools to collect info\\
- "non\_interactive": thinking, state-modifying actions, providing final answers\\
Respond as a JSON array of objects with keys: "action\_idx" (int), "label" (str), "justification" (str).
\end{tcolorbox}

\begin{tcolorbox}[
  title=Example System Prompt for Evidentiality Classification,
  colback=black!5!white,
  colframe=black!60!white,
  fonttitle=\bfseries,
  breakable
]
You are evaluating whether a customer service agent's actions are grounded \\
in the conversation context and domain policy.\\
\\
For each assistant action, determine:\\
- Are tool call arguments traceable to user-provided information or prior tool results?\\
- Are text claims about policies consistent with the provided policy document?\\
- Are stated facts (account details, prices, statuses) confirmed by tool results in the conversation?\\
\\
Labels:\\
- "evidential": The action is well-grounded in the conversation and policy.\\
- "non\_evidential": The action contains fabricated information, hallucinated \\
tool arguments, or claims contradicting the policy.\\
- "uncertain": Cannot determine grounding with available context.\\
\\
Respond as a JSON array of objects with keys:\\
"action\_idx" (int), "label" (str), "justification" (str), "issues" (list of str, empty if none).
\end{tcolorbox}

\paragraph{Uncertainty estimation.} In the proposed total uncertainty lower bound Eq.~\ref{eq:infogate}, there are three uncertainty terms, including initial query uncertainty, action uncertainty, and observation uncertainty. The \textit{action uncertainty} can be estimated or approximated by leveraging existing techniques, such as sampling-based ones entropy~\citep{malinin2021uncertainty,kuhn2023semantic}, pure probability-based methods~\citep{marina2020unsupervised,duan2024shifting}, hybrid~\citep{vashurin2025uncertainty}, or verbalized confidence methods~\citep{tian2023just,xiong2024can,yang2024verbalized}. However, as mentioned in Section~\ref{sec:challenges:measure}, all these existing approaches have their own limitations, and one may have to choose an appropriate estimator depending on the situation.

Meanwhile, \textit{observation uncertainty} and \textit{initial query uncertainty} will be much harder to estimate, as we usually do not have full knowledge of the world (See Section~\ref{sec:challenges:obs-uq}). If the system allows multiple trials, one can utilize a sampling-consistency-based approach to estimate the observation uncertainty, likewise the traditional black-box UQ methods. If the system does not allow this, one may try to construct a datastore that contains some relevant data from a similar task and leverage nonparameter estimation methods~\citep{ren2023outofdistribution,kotelevskii2022nonparametric}. Another line of methods can be establishing a world model to approximate the environment transition and estimate observation uncertainty from it~\citep{chae2025web,zhou2025dinowm,anonymous2026dreamphase}. Figure~\ref{fig:obs_uctt_full} provides results on an auxiliary LLM-based approximation that mimics world modeling philosophy.

\paragraph{Mutual information estimation.} The proposed information gating mechanism features information reduction based on the amount of information gain by measuring the conditional mutual information between the current observation $O_i$ and the initial query $O_0$, given all previous conversations except the initial query $E_{i}\backslash O_0$. While the estimation of mutual information is notorious for its difficulty, recent advancements on neural estimators~\citep{belghazi2018mutual,cheng2020club,mukherjee2020ccmi,molavipour2020conditional,gritsai2025mist} shed light on this problem by enabling estimation in high-dimensional, unstructured data space. Besides, LLM-prompting-based mutual information estimation~\citep{robertson2025measure} is also emerging as an attractive line of work, though it has yet to provide a solid theoretical foundation.

\paragraph{Conditional certainty maximization process.} As an analogy to the conditional uncertainty reduction process presented before, we can also envision a ``\textit{certainty maximization}'' approach by replacing the uncertainty with a certainty measure, such as informational energy~\citep{pardo1991information}, mentioned in Appendix~\ref{sec:apdx:uncertainty}. Given that Onicescu's informational energy brings its correlation coefficient as well, we may be able to implement the exact same style information-gating-based method that has the opposite direction to uncertainty reduction dynamics.

\subsection{Theoretical Analysis} \label{sec:apdx:action-conditional:proof}
\begin{lemma}[Extrema of Information Gating]\label{lem:apdx:extreme} Let the lower bound of agent total uncertainty in Eq.~\ref{eq:infogate} be $\tilde{U}(\mathcal{F}_{\leq T})$, denote $U(X):=H(X)=\mathbb{E}[-\log P(X)]$ and $\text{Info}(X;Y):=I(X;Y)=\mathbb{E}[\log\frac{P(X,Y)}{P(X)P(Y)}]$, then, we have:
\begin{align*}
    \tilde{U}(\mathcal{F}_{\leq T})&\geq H(E_0,O_0)-\sum_{i=1}^{T-1} I(O_i,O_0|E_{i}\backslash O_0) +H(A_T|E_{T-1},O_{T-1}),\\
    \tilde{U}(\mathcal{F}_{\leq T})&\leq H(E_0,O_0)+\sum_{i=1}^{T} H(A_i,O_i|E_{i-1},O_{i-1}).
\end{align*}
\end{lemma}
\begin{proof}   
    For $i$, if $A_{i} \in \mathcal{A}^{-}$ (resp. $A_{i} \notin \mathcal{A}^{-}$), turn-level uncertainty $U(\mathcal{F}_{i})$ becomes $I(O_i,O_0|E_{i}\backslash O_0)$ (resp. {\small$U(\mathcal{F}_{i})=U(A_i|E_{i-1},O_{i-1})+U(O_i|A_i,E_i)=H(A_i,O_i|E_{i-1},O_{i-1})$}). Thus, if all the intermediate actions are interactive and evidential, e.g., $A_{i} \in \mathcal{A}^{-}$ for all $1\leq i \leq T-1$ (resp. $A_{i} \notin \mathcal{A}^{-}$ for all $1\leq i \leq T-1$), our total uncertainty lower bound $\tilde{U}(\mathcal{F}_{\leq T})$ becomes a monotonic uncertainty reduction process (resp. monotonic propagation process), deriving the above two inequalities. Here, the final termination-turn action should not be interactive, and the observation in this turn is deterministic in our considered setups (Fig.~\ref{fig:pgm}), resulting in the accumulation of last action uncertainty $H(A_{T}|E_{T-1},O_{T-1})$.
\end{proof}

\subsection{Action Categorization} \label{sec:apdx:action-conditional:action}
\begin{table*}[t]
\centering
\resizebox{\linewidth}{!}{%
\small
\begin{tabular}{p{3.3cm}p{2cm}p{9cm}}
\toprule
\textbf{Action Category} & \textbf{Interactivity} & \textbf{Example} \\
\midrule
\textbf{Information-gathering} & interactive & Agent uses a read tool (e.g., get\_reservation\_details or search\_flights) to retrieve data or asks the user for missing information, e.g., “Could you provide your reservation number and last name?”.\\
\midrule
\textbf{Asking clarification or confirmation to user} & interactive  & Agent asks the user preferences or confirmation on decision, such as, “Do you want me to proceed with booking FQ8APE?”.\\
\midrule
\textbf{Thinking} & non-interactive  & Agent plans the future action sequence based on the previous trajectory.\\ \midrule
\textbf{State-changing tool call (writing)} & non-interactive  & Agent calls the tools that modify the database (e.g., cancel\_reservation, book\_reservation, update\_reservation\_flights) commit to an outcome.\\ \midrule
\textbf{Providing final information to user} & non-interactive  & Agent reports the result of an action, such as “Your reservation has been cancelled; your refund will be processed” or “Your flight has been rebooked to SFO departing at 8 a.m.”.\\ \bottomrule
\end{tabular}}
\caption{\textbf{Example action classes in Airline booking scenario in $\tau^2$-bench.} A number of the agent's actions can be categorized into five classes, which may reduce or increase the total uncertainty. See Table~\ref{tab:action_taxonomy} for extended analyses.} 
\label{tab:action_taxonomy_tbench_air}
\end{table*}
\begin{table*}[!t]
\resizebox{\linewidth}{!}{%
\small
\begin{tabular}{p{3.3cm}p{2cm}p{9cm}}
\toprule
\textbf{Action Category} & \textbf{Interactivity} & \textbf{Scenario-specific Example} \\
\midrule
\textbf{Information-gathering} & interactive & Search flight information ($\tau^2$-bench Airline); retrieve order status ($\tau^2$-bench Retail); read message box (ToolSandbox)\\
\midrule
\textbf{Asking clarification or confirmation to user} & interactive & Ask for flight choice or request final booking decision ($\tau^2$-bench Airline); instruct or ask the user to choose cancellation ($\tau^2$-bench Retail); ask for providing contact information or request updating contact number (ToolSandbox) \\
\midrule
\textbf{Thinking} & non-interactive & Plan flight schedule ($\tau^2$-bench Airline); consider returning items ($\tau^2$-bench Retail); think the reason of failed temperature checking (ToolSandbox) \\ \midrule
\textbf{State-changing tool call (write)} & non-interactive & Update flight booking ($\tau^2$-bench Airline); return/cancel items ($\tau^2$-bench Retail); Turn on Wifi (ToolSandbox)\\ \midrule
\textbf{Providing final information to user} & non-interactive & Summarize updated flights ($\tau^2$-bench Airline); summarize returned items ($\tau^2$-bench Retail); summarize phonebook update results (ToolSandbox)\\ \bottomrule
\end{tabular}}
\caption{\textbf{Example action classes in three scenarios from two benchmarks, $\tau^2$-bench~\citep{yao2024tau,barres2025tau} and ToolSandbox~\citep{lu2025toolsandbox}.}} 
\label{tab:action_taxonomy}
\end{table*}
An actual example categorization of the action interactivity in an airline assistant task is provided in Table~\ref{tab:action_taxonomy_tbench_air}. Besides, a higher abstraction of action category to two more scenarios, retail in $\tau^2$-bench~\citep{yao2024tau,barres2025tau} and ToolSandbox~\citep{lu2025toolsandbox} is provided in Table~\ref{tab:action_taxonomy}.

\section{Experiment} \label{sec:apdx:exp}
To ground our statements with empirical results, we conducted a small-scale pilot experiment on a real-world agent benchmark in general assistance tasks to evaluate existing UQ techniques on LLM agents. 

\subsection{Setup} \label{sec:apdx:exp:setup}
\paragraph{Dataset and task.}
$\tau^2$-bench~\citep{barres2025tau} (the latest version of $\tau$-bench~\citep{yao2024tau}) is a representative agent benchmark adopted in frontier LLM evaluation~\citep{gemini_team2026gemini31pro}, consisting of three different domains -- airline, retail, and telecom, where the total number of sample tasks is 50, 114, and 114, respectively. By following~\cite{gemini_team2026gemini31pro}, we experimented with retail and telecom domains. The tasks in the Retail domain are about helping users cancel or modify pending orders, return or exchange delivered orders, modify user addresses, or provide information. The tasks in the telecom domain are about tackling telecommunication issues in mobile devices, organized into roughly three categories: service issues, mobile data issues, and MMS issues.

\paragraph{LLMs and inference details.} We experiment with two modern LLMs: GPT-4.1~\citep{openai2025gpt41} and Kimi-K2.5~\citep{kimi2026} through Microsoft Azure\footnote{\url{https://azure.microsoft.com/en-us}} AI Foundry API endpoints. Most modern proprietary LLMs do not provide the output log probability in multiple inference engine consistency, and GPT-4.1 was a viable choice for us to adopt here. Meanwhile, small-scale open-source models that can be run in local GPUs did not achieve meaningful performance on $\tau^2$-bench due to its task complexity, so we stick with a state-of-the-art large-scale open-source model, Kimi-K2.5, through third party platform. We set the temperature to 0.0 during inference, which is the default value of $\tau^2$-bench. Due to the resource constraint, we only run a single trial per task in contrast to the four trials conducted in $\tau^2$-bench. User simulator LLM was fixed to Kimi-K2.5 for all tasks, and we only varied agent LLM between GPT-4.1 and Kimi-K2.5.

\paragraph{Implementation of uncertainty estimators.} We adopt three different uncertainty/confidence estimators: negative log-likelihood (NLL), Entropy, and verbalized confidence. The NLL and Entropy are calculated per token and averaged across the entire sequence of actions in the trajectory, whereas the verbalized confidence is calculated per turn and also averaged across the entire trajectory. To be specific, 
\begin{itemize}
    \item $\hat{U}_{\text{NLL}}(\mathcal{F}_{\leq T})=\frac{1}{T}\sum_{i=1}^{T}\hat{U}_{\text{NLL}}(A_i|E_{i-1},O_{i-1})=\frac{1}{T}\sum_{i=1}^{T}\frac{1}{L_i}\sum_{j=1}^{L_i}-\log P^{1}_{\pi,\mathcal{T}}(A_{i,j}|E_{i-1},O_{i-1})$
    \item $\hat{U}_{\text{Ent}}(\mathcal{F}_{\leq T})=\frac{1}{T}\sum_{i=1}^{T}\hat{U}_{\text{Ent}}(A_i|E_{i-1},O_{i-1})=\frac{1}{T}\sum_{i=1}^{T}\frac{1}{L_i}\sum_{j=1}^{L_i}\mathbb{E}_{A_{i,j}}[-\log P^{5}_{\pi,\mathcal{T}}(A_{i,j}|E_{i-1},O_{i-1})]$
    \item $\hat{C}_{\text{verb}}(\mathcal{F}_{\leq T})=\frac{1}{T}\sum_{i=1}^{T}\hat{C}_{\text{verb}}(A_i|E_{i-1},O_{i-1},\texttt{vc-prompt})$
\end{itemize}
where $L_i$ denotes the sequence length of $i$-th turn, $P^{K}$ indicates top-K output token distribution (a chosen token probability if $K=1$), and \texttt{vc-prompt} is a prompt for verbalized confidence we designed as below,

\begin{tcolorbox}[
  title=System Prompt for Verbalized Confidence~\citep{tian2023just},
  colback=black!5!white,
  colframe=black!60!white,
  fonttitle=\bfseries,
  breakable
]
Provide the probability (0.0 to 1.0) that your current response/action is correct for the given conversation history and context so far. Give ONLY the probability, no other words or explanation.
\end{tcolorbox}

\paragraph{Evaluation details.} Since $\tau^2$-bench does not provide turn-level annotation, we conduct trajectory-level evaluation that compares the trajectory-level uncertainty $U(\mathcal{F}_{\leq T})$ with the final task failure $1-r$ (or success $r$ in the case of verbalized confidence). For this, we measure the area under the receiver operating characteristic curve (AUROC), Spearman rank correlation coefficient $\rho$, and Kendall rank correlation coefficient $\tau$. The average reward (success rate) per model-domain pair is also reported in Table~\ref{tab:uq_eval}.

\subsection{Extended Results} \label{sec:apdx:exp:result}
This subsection provides full analysis results on the observation uncertainty in Figure~\ref{fig:obs_uctt_full} and uncertainty evolution in Figure~\ref{fig:uncertainty_evo_gpt_full} and Figure~\ref{fig:uncertainty_evo_kimi_full}. In Figure~\ref{fig:obs_uctt_full}, we compare three different observation uncertainties over all user messages across trajectories: (1) the GT user simulator uncertainty (computed from the output probability of user simulator LLM), (2) estimates driven from the output probability of agent LLM by feeding user messages along with all previous context, and (3) estimates driven from the output probability of an auxiliary LLM by feeding user messages along with all previous context, coupled with a prompt as below.

\begin{tcolorbox}[
  title=System Prompt for Observation Distribution Approximator,
  colback=black!5!white,
  colframe=black!60!white,
  fonttitle=\bfseries,
  breakable
]
You are an impartial observer monitoring a customer-service conversation between an AI agent and a human user. Your task is to predict the next observation that the agent will receive:\\
  1. A user message - the customer's natural-language reply, or\\
  2. A tool result  - the deterministic output returned by a backend API after the agent issues a tool call.\\
\\
Given the full conversation history up to this point, generate the most likely continuation as if you were the user or the tool backend. Aim to model the distribution of plausible observations as faithfully as possible: assign high probability to expected responses and low probability to surprising ones.
\end{tcolorbox}

\begin{figure*}
    \centering
    \includegraphics[width=\linewidth]{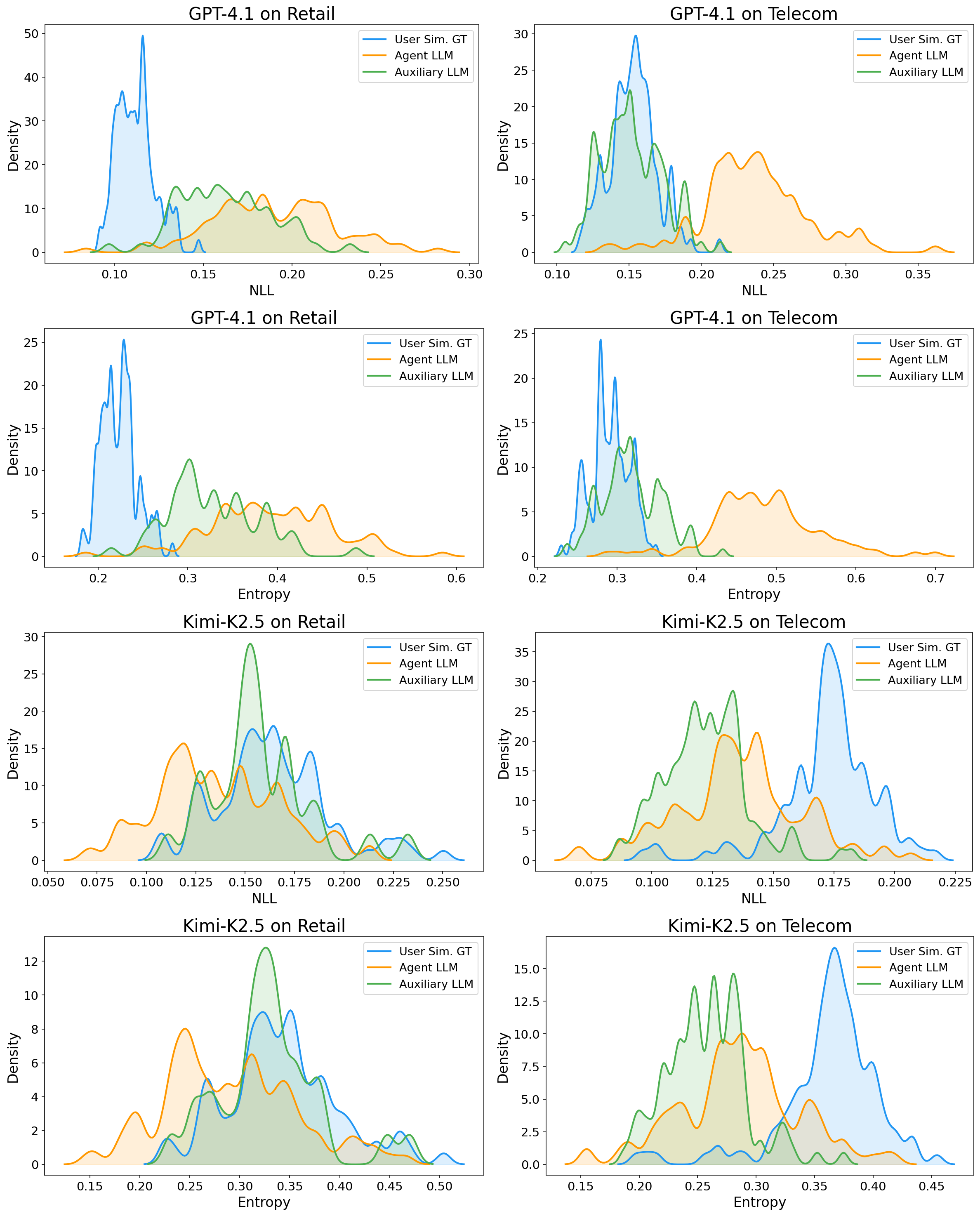}
    \caption{\textbf{Comparison of observation uncertainty estimates on $\tau^{2}$-bench.} We compare distributions over observation uncertainty estimates, i.e., average NLL and Entropy over all text message tokens from the user. Distributions of the ground truth user simulator LLM and that of the agent-approximated one show remarkable deviation, whereas auxiliary LLM prompting-based estimation shows somewhat close distribution with the ground truth.}
    \label{fig:obs_uctt_full}
\end{figure*}
We observe remarkable deviation between the GT user simulator's uncertainties and the estimates from the agent LLM in general. However, using an auxiliary LLM as a world model somehow reduces the gap in many cases. Cost-effective auxiliary model-based approximation can be worth exploring.

Next, Figure~\ref{fig:uncertainty_evo_gpt_full} and~\ref{fig:uncertainty_evo_kimi_full} compare uncertainty evolutions over turns between success and failure groups of tasks. To compare varying-length trajectories across different tasks, we normalize the turn index in X-axis into a zero-to-one scale. Results show that existing naive uncertainty aggregation methods can not discern the failure group from the success group to a statistically meaningful degree (the shaded areas in the plots denote a single standard deviation confidence interval), both in the middle of and at the end of the trajectory. 

\begin{figure*}[h]
    \centering
    \includegraphics[width=\linewidth]{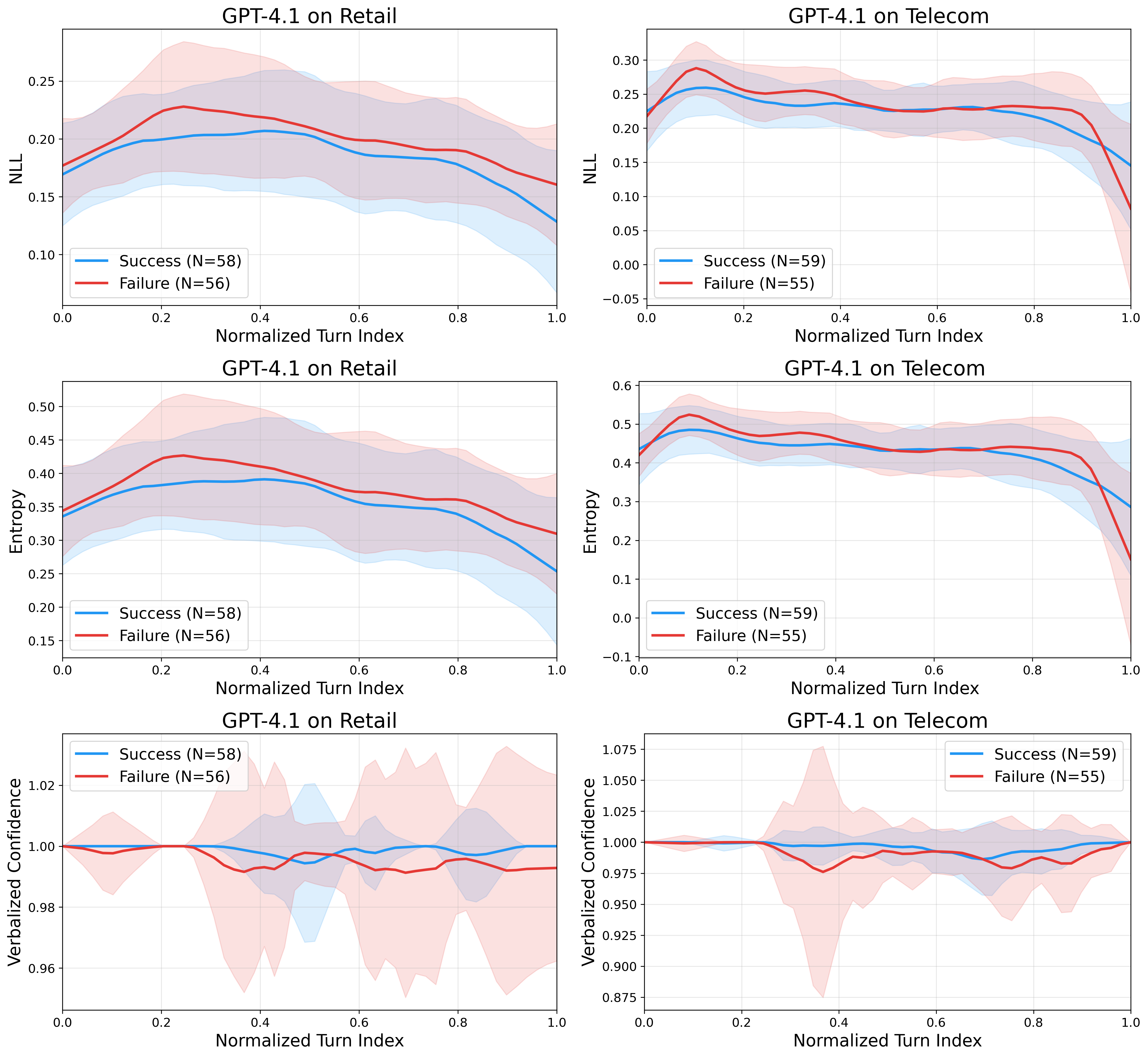}
    \caption{\textbf{Uncertainty evolution of GPT-4.1 across the trajectory on $\tau^{2}$-bench.} We visualize the evolution of action uncertainty/confidence estimates over the normalized turn index, averaged over the entire trajectories from two groups: task success and failure. The shade area denotes one standard deviation.}
    \label{fig:uncertainty_evo_gpt_full}
\end{figure*}
\begin{figure*}
    \centering
    \includegraphics[width=\linewidth]{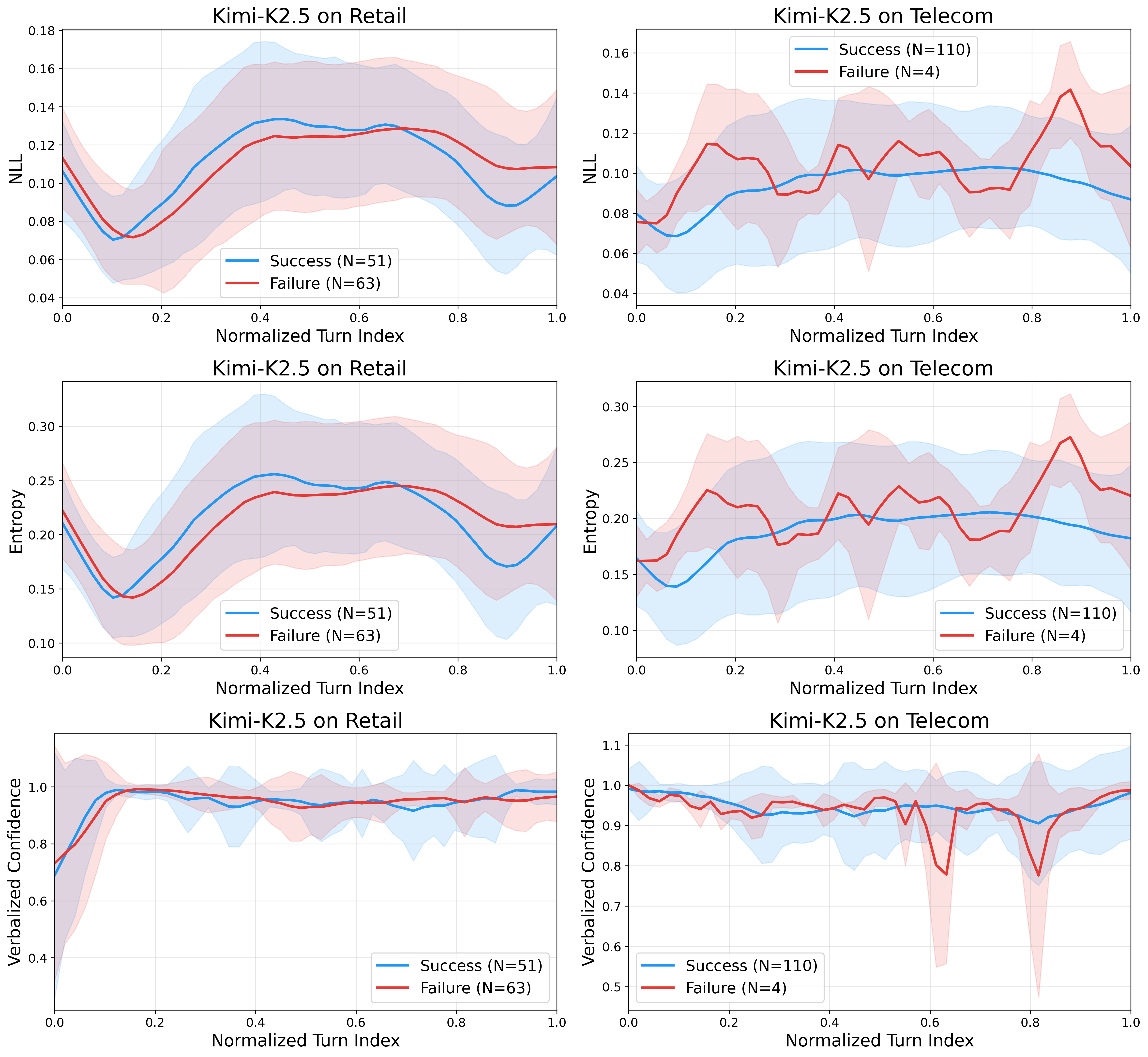}
    \caption{\textbf{Uncertainty evolution of Kimi-K2.5 across the trajectory on $\tau^{2}$-bench.} We visualize the evolution of action uncertainty/confidence estimates over the normalized turn index, averaged over the entire trajectories from two groups: task success and failure. The shade area denotes one standard deviation.}
    \label{fig:uncertainty_evo_kimi_full}
\end{figure*}

\clearpage
\newpage
\section{Mini-Survey on Agent Benchmark} \label{sec:apdx:survey}
We did a small-scale survey on LLM-based agent benchmarks released to the public from February 2023 to February 2026, where we sourced the papers mostly from \texttt{arxiv.org} and also referred to ML/NLP conference proceedings. We ended up with a non-exhaustive collection of 44 papers, mainly designed for pure LLM agents, while a few of them also support multimodal scenarios of VLM agents. We confine the scope of the survey to the agent's problem-solving capability rather than other desired properties such as safety and trustworthiness~\citep{zhang2024agent,levy2024st,andriushchenko2025agentharm}. We categorize the papers into three classes depending on the evaluation granularity: turn-level, milestone-level, and trajectory-level. Table~\ref{tab:survey_1} summarizes the turn-level and milestone-level evaluation benchmarks, and Table~\ref{tab:survey_2} summarizes the trajectory-level ones. The scarcity of fine-grained benchmarks is severe and becomes a real bottleneck to establishing and evaluating the agent UQ framework.
\begin{table*}[!t]
\resizebox{\linewidth}{!}{%
\scriptsize
\begin{tabular}{p{9.25cm}p{3cm}p{2cm}}
\toprule
\textbf{Title} & \textbf{Citation} & \textbf{Evaluation} \\
\midrule
Agent-as-a-Judge: Evaluate Agents with Agents & \cite{pmlr-v267-zhuge25a} & Turn-level \\  
API-Bank: A Comprehensive Benchmark for Tool-Augmented LLMs & \cite{li-etal-2023-api} & Turn-level \\ 
Clembench: Using Game Play to Evaluate Chat-Optimized Language Models as Conversational Agents & \cite{chalamalasetti2023clembench} & Turn-level \\  
Mind2Web: Towards a Generalist Agent for the Web & \cite{deng2023mind2web} & Turn-level \\ \midrule
AgentBoard: An Analytical Evaluation Board of Multi-turn LLM Agents & \cite{chang2024agentboard} & Milestone-level \\ 
DevBench: A Comprehensive Benchmark for Software Development & \cite{li2025prompting} & Milestone-level \\ 
Gaia2: Benchmarking LLM Agents on Dynamic and Asynchronous Environments & \cite{froger2026gaia2} & Milestone-level \\ 
PaperBench: Evaluating AI's Ability to Replicate AI Research & \cite{pmlr-v267-starace25a} & Milestone-level \\ 
ScienceAgentBench: Toward Rigorous Assessment of Language Agents for Data-Driven Scientific Discovery & \cite{chen2025scienceagentbench} & Milestone-level \\ 
ST-WebAgentBench: A Benchmark for Evaluating Safety and Trustworthiness in Web Agents & \cite{levy2024st} & Milestone-level \\ 
The Berkeley Function Calling Leaderboard (BFCL): From Tool Use to Agentic Evaluation of Large Language Models & \cite{patil2025berkeley} & Milestone-level \\ 
TheAgentCompany: Benchmarking LLM Agents on Consequential Real World Tasks & \cite{xu2025theagentcompany} & Milestone-level \\ 
ToolSandBox: A Stateful, Conversational, Interactive Evaluation Benchmark for LLM Tool Use Capabilities & \cite{lu2025toolsandbox} & Milestone-level \\ 
WebCanvas: Benchmarking Web Agents in Online Environments & \cite{pan2024webcanvas} & Milestone-level \\ 
\bottomrule
\end{tabular}}
\caption{\textbf{Agent benchmark research with turn-level and milestone-level evaluation protocols.}} 
\label{tab:survey_1}
\end{table*}
\begin{table*}[!t]
\resizebox{\linewidth}{!}{%
\scriptsize
\begin{tabular}{p{11cm}p{3.5cm}}
\toprule
\textbf{Title} & \textbf{Citation} \\
\midrule
AgentBench: Evaluating LLMs as Agents & \cite{liu2024agentbench} \\
AgentGym: Evolving Large Language Model-based Agents across Diverse Environments & \cite{xi-etal-2025-agentgym} \\
AppWorld: A Controllable World of Apps and People for Benchmarking Interactive Coding Agents & \cite{trivedi-etal-2024-appworld} \\
AssistantBench: Can Web Agents Solve Realistic and Time-Consuming Tasks? & \cite{yoran2024assistantbench} \\
BrowseComp: A Simple Yet Challenging Benchmark for Browsing Agents & \cite{wei2025browsecomp} \\
CORE-Bench: Fostering the Credibility of Published Research Through a Computational Reproducibility Agent Benchmark & \cite{siegel2024corebench} \\
Finance Agent Benchmark: Benchmarking LLMs on Real-world Financial Research Tasks & \cite{bigeard2025finance}  \\
GAIA: A Benchmark for General AI Assistants & \cite{mialon2023gaia} \\
GameBench: Evaluating Strategic Reasoning Abilities of LLM Agents & \cite{costarelli2024gamebench} \\
InterCode: Standardizing and Benchmarking Interactive Coding with Execution Feedback & \cite{yang2023intercode} \\
MCPAgentBench: A Real-world Task Benchmark for Evaluating LLM Agent MCP Tool Use & \cite{liu2025mcpagentbench} \\
MCP-Atlas: A Large-Scale Benchmark for Tool-Use Competency with Real MCP Servers & \cite{bandi2026mcpatlas} \\
MedAgentBench: A Realistic Virtual EHR Environment to Benchmark Medical LLM Agents & \cite{jiang2025medagentbench} \\
MINT: Evaluating LLMs in Multi-turn Interaction with Tools and Language Feedback & \cite{wang2024mint} \\ 
MLAgentBench: Evaluating Language Agents on Machine Learning Experimentation & \cite{pmlr-v235-huang24y} \\
OSWorld: Benchmarking Multimodal Agents for Open-Ended Tasks in Real Computer Environments & \cite{xie2024osworld} \\
SafeAgentBench: A Benchmark for Safe Task Planning of Embodied LLM Agents & \cite{yin2024safeagentbench} \\
SmartPlay: A Benchmark for LLMs as Intelligent Agents & \cite{wu2024smartplay} \\
SOTOPIA: Interactive Evaluation for Social Intelligence in Language Agents & \cite{zhou2024sotopia} \\
SWE-bench: Can Language Models Resolve Real-World GitHub Issues? & \cite{jimenez2024swebench} \\
Terminal-Bench: Benchmarking Agents on Hard, Realistic Tasks in Command Line Interfaces & \cite{merrill2026terminal} \\
The Tool Decathlon: Benchmarking Language Agents for Diverse, Realistic, and Long-Horizon Task Execution & \cite{li2025tooldecathlon} \\
ToolLLM: Facilitating Large Language Models to Master 16000+ Real-world APIs & \cite{qin2024toolllm} \\
TravelPlanner: A Benchmark for Real-World Planning with Language Agents & \cite{xie2024travelplanner} \\
Vending-Bench: A Benchmark for Long-Term Coherence of Autonomous Agents & \cite{backlund2025vending} \\
WebArena: A Realistic Web Environment for Building Autonomous Agents & \cite{zhou2024webarena} \\
WebShop: Towards Scalable Real-World Web Interaction with Grounded Language Agents & \cite{yao2022webshop} \\
WorkArena: How Capable Are Web Agents at Solving Common Knowledge Work Tasks? & \cite{drouin2024workarena} \\
$\tau$-bench: A Benchmark for Tool-Agent-User Interaction in Real-World Domains & \cite{yao2024tau} \\
$\tau^2$-Bench: Evaluating Conversational Agents in a Dual-Control Environment & \cite{barres2025tau} \\
\bottomrule
\end{tabular}}
\caption{\textbf{Agent benchmark research with a trajectory-level evaluation protocol.}} 
\label{tab:survey_2}
\end{table*}

\section{Discussion and Future Work} \label{sec:apdx:disc}
Does solely establishing the agent UQ framework actually help us to reliably infer on the agent's performance on a task? Modern LLMs (especially after post-training) are not well-calibrated~\citep{tian2023just}, so the uncertainty estimates from these ill-calibrated probabilistic models can not be directly used as a reliable performance indicator. Although we present a high-level roadmap for quantifying uncertainty for LLM agent systems, future work should also explore the calibration of LLM agents as well. The connection between calibration and accuracy~\citep{park2020calibrated,oh2024towards} of a predictor shows an exciting future work direction for joint optimization of the agent's problem-solving capability and calibration simultaneously.

Another promising direction for future investigations would be to extend the agent UQ framework to multimodal setups such as graphical user interface agents~\citep{nguyen-etal-2025-gui}, where extra challenges occur, such as many-to-many correspondence between modalities~\citep{chun2025multiplicity}. Defining and establishing a reliable UQ framework for multimodal agents would offer helpful foundations to develop robustness solutions~\citep {oh2025understanding,oh2025visual}, safety~\citep{zhou2025multimodal,chen2025obvious} and more~\citep{shi2025towards}.
\clearpage

\end{document}